\RequirePackage{fix-cm}
\documentclass[twocolumn]{svjour3}          % twocolumn
\smartqed  % flush right qed marks, e.g. at end of proof
\usepackage{graphicx}
\usepackage[italian,english]{babel}
\usepackage[latin1]{inputenc}
\usepackage{fancyhdr}
\usepackage{indentfirst}
\usepackage{graphicx}
\usepackage{newlfont}
\usepackage{natbib}
\usepackage{amssymb}
\usepackage{amsmath}
\usepackage{latexsym}
\usepackage{float}
\usepackage{subfig}
\usepackage{caption}
\usepackage{hyperref}
\usepackage{authblk}
\usepackage[margin=1.375in]{geometry}
\usepackage[table]{xcolor}
\newtheorem{teo}{Theorem}[section]      %definizione ambiente teorema
\newtheorem{rem}{Remark}
\usepackage{Bertonithesis}  

\begin{document}

\title{LGN-CNN: a biologically inspired CNN architecture%\thanks{This project has received funding from the European Union's Horizon 2020 research and innovation program under the Marie Sk\l odowska-Curie grant agreement No 754362. %{\color{green} Inserire FLAG.}
%\includegraphics[height=1cm]{flag_EU_per_paper.jpg}
 %}
}


\author{Federico Bertoni         \and
        Giovanna Citti \and Alessandro Sarti %etc.
}

% \institute{This project has received funding from the European Union's Horizon 2020 research and innovation program under the Marie Sk\l odowska-Curie grant agreement No 754362. %{\color{green} Inserire FLAG.}
%\includegraphics[height=1cm]{flag_EU_per_paper.jpg}}


\institute{$^*$This project has received funding from the European Union's Horizon 2020 research and innovation program under the Marie Sk\l odowska-Curie grant agreement No 754362. \includegraphics[height=1cm]{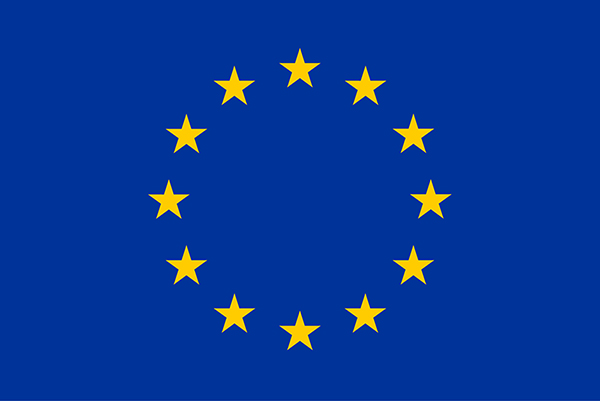} \\
$^{**}$This project has received funding from GHAIA - Marie Sk\l odowska-Curie grant agreement No 777822.%{\color{green} Inserire FLAG.}
\\
\hrule \vspace{2mm}
Federico Bertoni$^*$ \at
              Sorbonne Universit\'e, Paris, France \\
              \email{federico.bertoni4@unibo.it}           %  \\
           \and
           Giovanna Citti$^{**}$ \at
              Dipartimento di Matematica, Universit\`{a} di Bologna, Italy. \\
              \email{giovanna.citti@unibo.it}  
            \and
           Alessandro Sarti$^{**}$ \at
              CAMS, CNRS - EHESS, Paris, France.  \\
              \email{alessandro.sarti@ehess.fr}  \\
}

\date{}

\maketitle

\section*{Abstract} \label{Abstract}            %crea il capitolo
%%%%%%%%%%%%%%%%%%%%%%%%%%%%%%%%%%%%%%%%%imposta l'intestazione di pagina
\lhead[\fancyplain{}{\bfseries\thepage}]{\fancyplain{}{\bfseries\rightmark}}
%\pagenumbering{arabic} 

In this paper we introduce a biologically inspired Convolutional Neural Network (CNN) architecture called LGN-CNN that has a first convolutional layer composed by a single filter that mimics the role of the Lateral Geniculate Nucleus (LGN). The first layer of the neural network shows a rotational symmetric pattern justified by the structure of the net itself that turns up to be an approximation of a Laplacian of Gaussian (LoG). The latter function is in turn a good approximation of the receptive field profiles (RFPs) of the cells in the LGN. The analogy with  the visual system  is established, emerging directly from the architecture of the neural network. A proof of rotation invariance of the first layer is given on a fixed LGN-CNN architecture and the computational results are shown. Thus, contrast invariance capability of the LGN-CNN is investigated and a comparison between the Retinex effects of the first layer of LGN-CNN and the Retinex effects of a LoG is provided on different images. A statistical study is done on the filters of the second convolutional layer with respect to biological data.
In conclusion, the model we have introduced approximates well the RFPs of both LGN and V1 attaining similar behavior as regards long range connections of LGN cells that show Retinex effects.

%and the second bank of filters is studied in comparison with the RPs of simple cells in V1.

%% Modificare l'abstract?? Se non cé unicitá non é un problema o meglio lo é ma possiamo dire che se C fa passare fuori la rotazione abbiamo comunque una soluzione

\section{Introduction} \label{Intro}            %crea il capitolo
%%%%%%%%%%%%%%%%%%%%%%%%%%%%%%%%%%%%%%%%%imposta l'intestazione di paginao
\lhead[\fancyplain{}{\bfseries\thepage}]{\fancyplain{}{\bfseries\rightmark}}
\pagenumbering{arabic}  
\subsection{Architecture of the visual system and CNNs}

 The visual system is composed by many cortices that elaborate the visual signal received from the retina via the optical nerve. 
Each cortex receives information from other cortices, processes it through horizontal connectivity, forward it to higher areas and send feedback to previous ones. %Even though higher areas process more complex features,  
The structure is very complex and not totally ordered as physiologically described for example in \cite{Hubel}. Geometrical model of the first visual cortex, we refer to \cite{Ferraro}, \cite{Petitot}, \cite{CittiSarti2006}.

The first neural nets have been inspired by a simplification of this structure,  and present a hierarchical structure, where each layer receives input from the previous one and provides output to the next one. Despite this simplification, they reached optimal performances in processes typical of the natural visual system,  as for example object-detection \cite{Redmon}, \cite{Ren} or image classification \cite{He}, \cite{Simonyan}.

More recently relations between CNNs and human visual system have been widely studied, with the ultimate scope of making the CNN even more efficient in specific tasks. 
A  model of the first cortical layers described as layers of a CNN has been studied in \cite{Poggio2}.
%Strict connections between each cortical layer and convolutional layers in a convolutional neural networks were studied in \cite{Poggio2} by fitting biological constraints. 
% \qq{Eliminata parte su Yamins}
In \cite{Yamins} and in \cite{Yamins2} the authors were able to study  higher areas by focusing on the encoding and decoding ability of the visual system. 
Recurrent Neural networks have been introduced to implement 
the horizontal connectivity (as for example in \cite{Sherstinsky}), 
or feedback terms (for example in \cite{Liang}). 
A modification of these nets, more geometric and more similar to the 
structure of the brain, have been recently proposed in \cite{Montobbio}.

It is well known that both V1 RFPs and the first convolutional layer of a CNN are 
 mainly composed by Gabor filters. 
We refer to \cite{Daugman}, \cite{Jones}, \cite{Lee}, \cite{Petitot} for the visual system 
and  to \cite{Yamins}, \cite{Yamins2}, \cite{Poggio1} for
% rotation invariance 
properties of CNNs.

% Even though 
Biological based models of V1 in terms of Gabor filters have been made in \cite{Zhang} and \cite{Poggio2} and
% (who was able to study each layer of the cortex up to V4), it is still missing a comparison between the statistics of learned filters and the statistics of V1 RFPs. In particular
the statistic of the RFPs of a macaque's V1 was studied in \cite{Ringach}, but a comparison between these results and the statistics of learned filters is still missing.

%{\color{black} A study of Gabor-based standard models for V1 with respect to the filters of CNNs has been done in \cite{Zhang}. In \cite{Poggio2} the authors introduced a biologically inspired architecture in which each layer is interpreted as a cortical layer of the first visual cortices. The action of each layer is modeled with the RPs of the cells with a-priori structure and the learning stage regards only a single layer in which the filters are selected among a set of prototypes.
%and gave to each layer a biological interpretation in terms of the actions of the cells in the visual system.
%A biologically inspired architecture whose filters are learned will lead us to perform a comparison between the statistics of learned filters and the statistics of the filters obtained on the RPs of a macaque's V1 from the work of Ringach (\cite{Ringach}) which is still missing. }
 
%{\color{green} NON ci sono referenze MENo recenti? citare?
%Convolutional neural network models of V1 responses to complex patterns
%Yimeng Zhang  Tai Sing Lee, Ming Li, Fang Liu, Shiming Tang

% Da per tutto c'è scritto che Poggio fitta i dati. controllare cosa fa e dire che non c'è ancora un confronto con Ringach However a comparison between the statistics of filters learned by a net and the statistics of the filters obtained on the RPs of a macaque's V1 from the work of Ringach (\cite{Ringach}) is still missing. 
%}

% \subsection{Invariance and equivariance properties in V1 and CNNs}
\subsection{Invariance properties in  CNNs}

Gabor invariance properties are mainly invariance with respect to translation and rotations. CNNs are translation equivariant since they are defined in terms of convolutional
kernels (see \cite{Cohen2016}, \cite{Cohen2018}). Rotation invariance properties can be imposed either obtaining the whole bank of filters from a learned one \cite{Marcos} and \cite{Wu}, or rotating any test image  \cite{Fasel}, \cite{Dieleman}, \cite{DielemanFauw}. A different kind of pooling or kernel procedure are used in \cite{Laptev} and \cite{Gens}, while \cite{Barnard} studied invariances with respect to other feature spaces.

\subsection{LGN, Retinex and contrast perception}

In the human visual system, the process in V1 operated by Gabor filters, is preceded by a preprocessing operated by radially symmetric families of cells, both in the retina and in the LGN (see \cite{Hubel}). 
%They are involved in many processes of the visual system (as discribed in ) and in particular the Magnocellular cells are the ones involved in contrast perception. We are mainly interested  on these cells and 
The RFPs of cells can be approximated by a LoG which is rotational symmetric (for a review see for example \cite{Petitot}).  It does not seem that an analogous layer is present in classical CNN, but it is known that it is crucial for human contrast perception. It has been investigated in Retinex theory formulated in 1964 by E. H. Land in \cite{Land1} and further developed by E. H. Land and J. McCann in \cite{Land2}.  Several developments are due to \cite{Brainard}, \cite{Provenzi} and \cite{Lei} among others. 
%New
An analysis of retinal and cortical components of the Retinex theory has been performed by several research teams in the past years (for a review see e.g. \cite{Yeonan}, \cite{Valberg}) and  it has been shown that the Magnocellular cells are the ones involved in contrast perception (see e.g. \cite{Enroth}, $\;\;\;\;\;\;\;\;\;$
\cite{Solomon}).
 Variational approaches have been proposed by \cite{Kimmel}, \cite{Morel1} and \cite{Morel2}. A geometrical model which makes a first step in relating the architecture of the visual system and invariance of RFPs has been presented in 
\cite{Gauge}. The action of radially symmetric RFP is interpreted as a LoG, while the horizontal connectivity is modeled as an inverse of the Laplacian operator, and allows to recover the given image up to a contrast shift.

\subsection{Our contribution}

In the present paper, we introduce a new neural architecture inspired by the structure of visual cortices
and study the properties of the filters of the first and second layers.
The first layer contains a single filter and models the LGN.  
We show that it has the same radially symmetric shape, as LGN receptive profiles and is able to reproduce a Retinex effect. 
Then we show that the statistics of filters of the second layer much better fits the observed distributions, comparing with the esperimental results of \cite{Ringach}.
The paper is organized as following.
% We show that it attains a LoG shape and we   study its  Retinex effects  comparing it with classical results of the Retinex theory. Furthermore, we study how its introduction influenced the subsequent layer. Indeed, it is well known (see e.g. \cite{Daugman}, \cite{Jones}) that the first layer of a CNN attains Gabor shape filters; however, adding the LGN layer  modifies the shape of Gabor filters attaining a more biological distribution.
% and deduce from it the invariance properties of the filters of the first layer.

In Section \ref{subs:VisualSystem} we recall  the structure of LGN and V1 and the RFPs of their cells. We described an interpretation of the Retinex model given by Morel \cite{Morel1} and a statistical study on the RFPs of the distributions of V1 cells in a macaque by Ringach \cite{Ringach}.

%{\color{green} dire cosa c'è nella section 2}

%{\color{black} 
%After recalling some properties of the visual system in section 2, }
In Section \ref{First} we introduce our  LGN-CNN architecture and describe it in detail. 
% It is  characterized by the presence of only one filter in the first convolutional layer. 
The new architecture is implemented in Section \ref{Second}. It provides classification performances comparable to the classical one, but it enforces the development of rotation symmetric filter during the training phase. 
As a result, the filter  is  a good approximation of RFP of 
LGN. 
% In particular, with this model we establish a relation between the architecture of the CNNs and the invariance properties of their filters, in order to find an architecture that gives rotational symmetric filter in the first convolutional layer. 

In Section \ref{subs:Gen_appr_Retinex} we test our filter on contrast perception phenomena. Indeed, we use the  learned kernel to repeat and justify the Retinex model.  
% We apply the radially symmetric learned kernel to an input image, then we apply its inverse (which can represent a feedback mechanism) and obtain a perceived image. 
We test the model on a number of classical Retinex illusion, comparing with the results of \cite{Morel1}.

Thus,  in Section \ref{subs:RingachStudy} we show that  the filters in the second layer of the net mimics the shape of Gabor filters of V1 RFPs. We have already recalled that this happens also in  standard CNN. We compare the statistics of the filters with the results obtained on the RFPs of a macaque's V1 from the work of Ringach \cite{Ringach}. By construction both in our net and in the visual system the second layer do not exhibit LoG, which are collected in the first layer. Also, other statistical properties of the filters are present in our LGN-CNN. 

The proof of the symmetry of the filter of the first layer is collected in appendix  \ref{Third}.
It is provided in the simplified case of a neural network composed by only a layer composed by a single filter.

\section{The visual system} \label{subs:VisualSystem}           %crea il capitolo
%%%%%%%%%%%%%%%%%%%%%%%%%%%%%%%%%%%%%%%%%imposta l'intestazione di pagina
\lhead[\fancyplain{}{\bfseries\thepage}]{\fancyplain{}{\bfseries\rightmark}}
%\pagenumbering{arabic}  

The visual system is one of the most studied part of the brain.
% , thus, we describe only the topics that interested the most our studies. For a complete overview see for example
We describe here the aspects important for our study and refer the reader for a more general description to \cite{Nolte}, \cite{Jessell}.
% \cite{Petitot}, \cite{CittiSarti2006} %{\color{green} citare citti sarti 2006}

The retina is a light-sensitive layer of tissue which receives the visual stimulus and translates it into electrical impulses. These impulses first reach the LGN  whose cells preprocess the visual stimulus. Then the impulse is processed by the cells of V1, whose output is taken in input to all the other layers of the visual system. 

%We are mainly interested in the cells of the LGN. Each cell receives the electrical impulse from a little portion of the retina $\Omega$ called receptive field (RF). The RF of each cell is divided in excitatory and inhibitory areas which are activated by the light and that can be modeled as a function $\Psi:\Omega \subset \R^2 \to \R$ called receptive profile (RP). Thus, if the excitatory areas are activated the cell will fire whereas it will be silenced in the case of inhibitory areas activation. Figure \ref{fig:LGN} shows the RP of a LGN cell that can be modeled by a laplacian of gaussian (LoG). Indeed it will highlight the contours of the objects in the image whereas it will set to zero all the uniform areas. 
%%Then the processed visual stimulus reaches V1 in which it is analyzed by its set of cells 
%
%On the other hand figure \ref{fig:Gabor} shows the RPs of two simple cells of V1 modeled by Gabor functions. These cells will fire if the contour of the object passing in their RF has a similar orientation w.r.t. their RPs.

We are mainly interested in the cells of the LGN and in the simple cells of V1. Each cell receives the electrical impulse from a  portion of the retina $\Omega$ called receptive field (RF). The RF of each cell is divided in excitatory and inhibitory areas which are activated by the light and that can be modeled as a function $\Psi:\Omega \subset \R^2 \to \R$ called receptive field profile (RFP). Thus, if the excitatory areas are activated the firing rate of the cell increases whereas it decreases in case of inhibitory areas activation. Figure \ref{fig:LGN} shows the RFP of an LGN cell that can be modeled by a LoG. 
% On the other hand, Figure \ref{fig:Gabor} shows the RFPs of two simple cells of V1, modeled by Gabor functions defined in eq. (\ref{eq:Gabor}). 

\begin{figure}
\centering
\includegraphics[height=3.4cm]{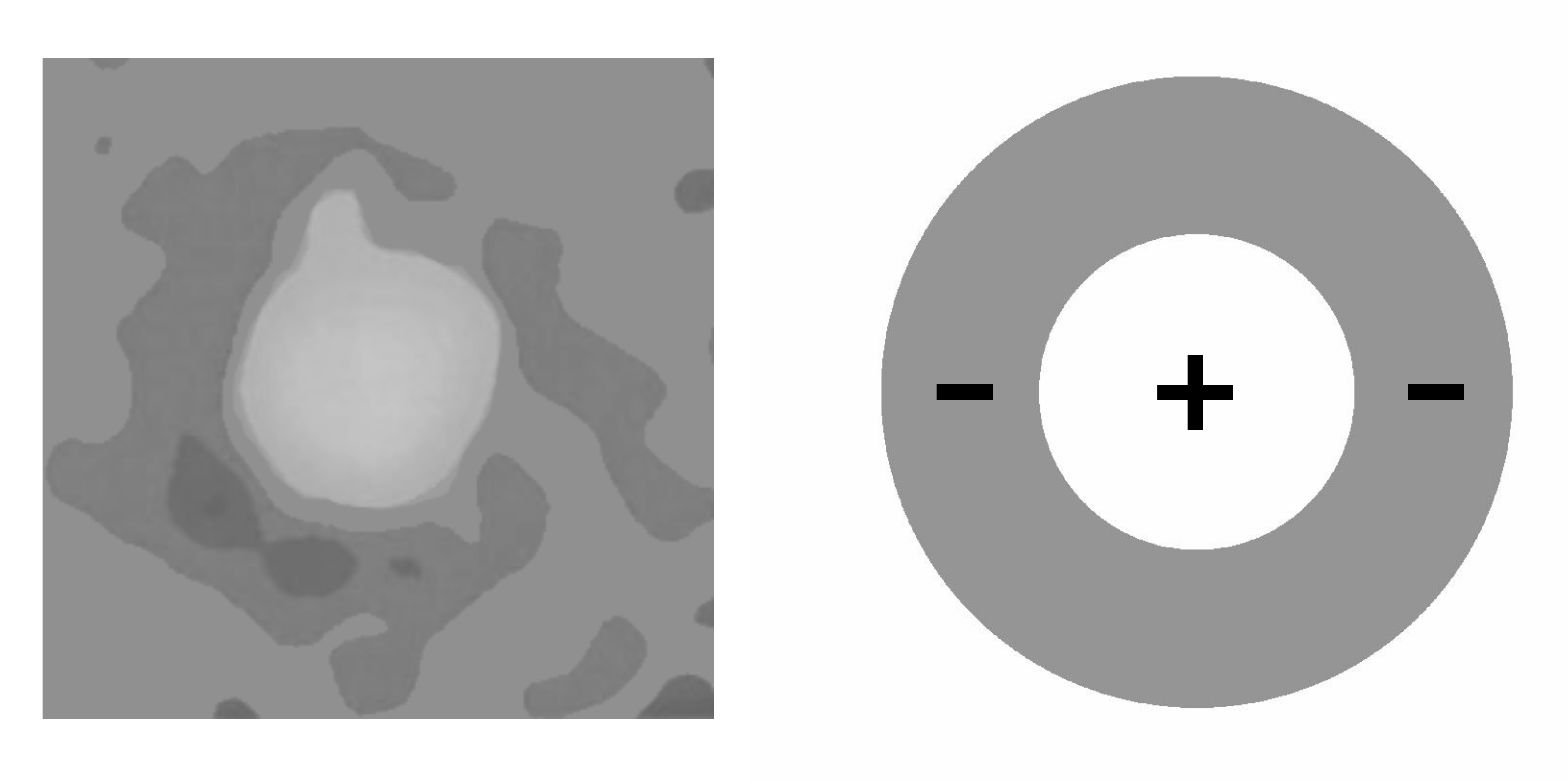}
\caption{On the left: RFP of an LGN cell where the excitatory area is in white and the inhibitory one is in gray. On the right: Its approximation by a LoG. \emph{From}: \cite{DeAngelis}. }
\label{fig:LGN}
\end{figure}

\begin{figure}
\centering
\includegraphics[height=3.8cm]{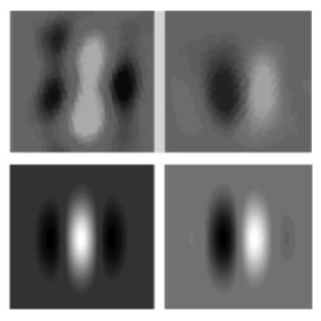}
\caption{First row: RFPs of two simple cells of V1 where the excitatory area is in white and the inhibitory one is in black. Second row: their approximations by Gabor functions. \emph{From}: \cite{SartiCitti}. }
\label{fig:Gabor}
\end{figure}

\subsection{LGN RFP: neural interpretation of Retinex model} \label{sec:Morel}            %crea il capitolo
%%%%%%%%%%%%%%%%%%%%%%%%%%%%%%%%%%%%%%%%%imposta l'intestazione di pagina
\lhead[\fancyplain{}{\bfseries\thepage}]{\fancyplain{}{\bfseries\rightmark}}
%\pagenumbering{arabic} 

The Retinex algorithm, introduced in \cite{Land2}, mimics the contrast invariant process performed by our visual system and associates to an image $I : \Xi \subset\R^2 \to \R$ the perceived image $\widetilde{I}$.

In \cite{Morel1} and \cite{Morel2} the authors have formalized the Retinex theory 
% obtaining a Retinex Poisson equation with Neumann boundary conditions that reproduces the color perception effects of the visual system. Introducing symmetric random walks as Retinex paths on the Retinex algorithm, they have shown that it can be simplified 
as the solution of the following discrete Poisson PDE 
\begin{equation}
\label{eq:PDE_Morel}
\;\;\;\;\; \;\;\;\;\; \;\;\;\;\; \;\;\;\;\; - \Delta_d \widetilde{I} = M(I)
\end{equation}

% where $I$ is the starting image defined as a function on  $\Xi \subset\R^2$ 
% \begin{equation*}
% \label{eq:I}
% \;\;\;\;\; \;\;\;\;\; \;\;\;\;\; \;\;\;\;\;  I : \Xi \subset\R^2 \to \R,
% \end{equation*}

where $\Delta_d$ is the classical discrete Laplacian and $M$ is a modified version of the discrete Laplacian.
% where at each difference between the central pixel and one in proximity is applied a threshold function. $\widetilde{I}$, defined as follows 
% \begin{equation*}
% \label{eq:Itilde}
% \;\;\;\;\; \;\;\;\;\; \;\;\;\;\; \;\;\;\;\;  \widetilde{I} : \Xi \subset\R^2 \to \R,
% \end{equation*}
% is the solution to the problem described by eq. (\ref{eq:PDE_Morel}) and it represents the reconstructed image from $I$ that shows the color perception effects expected from Retinex theory.

In \cite{Gauge} a neural interpretation of the model has been introduced. The RFP of the LGN cell which takes in input the visual signal acting by convolution on it,
% $I$. Calling $G_\sigma$ a Gaussian bell, the RFP is 
modeled as a modified discrete LoG, $M(G_\sigma) \approx \Delta G_\sigma$,
where $G_\sigma$ is a Gaussian bell
% and acts by convolution on the signal.
Hence the action on an input image is the following: 
$$Out_{LGN}(I) = \Delta (G_\sigma * I)  \approx \Delta I$$
The horizontal connectivity in this layer is radially symmetric and modeled as 
% the inverse operator of the 
% Laplace operator $\Delta^{-1}$. It is simply the operator of convolution with
the fundamental solution $\log ( \sqrt{x^2 + y^2})$ whose associated operator is the inverse of the Laplacian $\Delta^{-1}$ and  allows to recover the function $\widetilde{I}$:
$$\widetilde{I} = {\Delta^{-1}}* Out_{LGN}(I).$$
As a result 
$$\Delta \widetilde{I} =  Out_{LGN}(I) \approx \Delta I, $$
is the Retinex equation.  In general, $\widetilde{I}$ will not coincide with $I$, but 
will differ by a harmonic function. 

Our aim is to replace the action of the RFP with the filter learned by the LGN-CNN.
% face this problem for an operator $M$, learned by the filters. 
If it is a good approximation of the associated $\Delta G_\sigma$, then its inverse will allow to recover the perceived image  $\widetilde{I}$ in problems of contrast perception. 
In Section \ref{subs:Gen_appr_Retinex} we will describe in detail the process.

\subsection{Statistics of V1 RFPs} 
As first discovered by Daugmann, RFPs of the primary cortex V1 can be approximated by Gabor functions defined as follows:
\begin{equation}\label{eq:Gabor}
\begin{split}
 h(x',y') = & A e^{(-(x'/\sqrt{2} \sigma_x)^2 - (y'/\sqrt{2} \sigma_y)^2 )} \\ &  \cos (2 \pi f x' + \phi)
\end{split}
\end{equation}
where $(x',y')$ is translated and rotated from the original coordinate system $(x_0 , y_0)$
\begin{equation*}
\begin{split}
 x'= & (x-x_0) \cos \theta + (y- y_0) \sin \theta \\ y' = & - (x- x_0) \sin \theta + (y-y_0) \cos \theta.
\end{split}
\end{equation*}
as shown in Figure \ref{fig:Gabor}. 

%\begin{multline*}
% x'=(x-x_0) \cos \theta + (y- y_0) \sin \theta \\ y' = - (x- x_0) \sin \theta + (y-y_0) \cos \theta. \;\;\;\; \;\;\;\;\;
%\end{multline*}

Recently, Ringach in \cite{Ringach} has proved that RFPs are not uniformly distributed with respect to all the Gabor parameters, but they have a very particular statistic. 
% The approximated Gabor  filters 
% were obtained using the Independent Component Analysis and the Sparse Coding.
Ringach defines two coefficients $n_x$ and $n_y$ which estimate the elongation in $x$ and $y$ directions respectively
\begin{equation*}
    (n_x, n_y) = (\sigma_x \cdot f, \sigma_y \cdot f).
\end{equation*}
% $(n_x, n_y) = (\sigma_x \cdot f, \sigma_y \cdot f)$ where $n_x$ and $n_y$ estimate the elongation in $x$ and $y$ directions respectively. 
% They are rescaled by $f$ which indicates how far the shape of the filter is with respect to a Gaussian; 
In particular, if $f=0$ the function $h$ in (\ref{eq:Gabor}) simplifies to a Gaussian since the cosine becomes a constant. Otherwise it is elongated:
%  These are the main steps he followed:
\begin{itemize}
% \item Recording the RFPs from several simple cells in V1;
\item Fitting a Gabor function defined in equation (\ref{eq:Gabor}) to the RFPs;
\item Comparing the results on $(n_x, n_y) = (\sigma_x \cdot f, \sigma_y \cdot f)$ plane.
\end{itemize}

Figure \ref{fig:stat_distrC} shows the statistical distribution of RFPs of V1 cells in monkeys in $(n_x, n_y)$ plane obtained by Ringach in \cite{Ringach}. In \cite{Barbieri} the authors have studied the same statistical distribution with respect to the Uncertainty Principle associated to the task of detection of position and orientation.

%{\color{green}mettere un'immagine
%A model which justifies the statistical distribution has been provided by Barbieri  - citare}

%\input{LGN_Rx_2.tex}
\section{Introducing LGN-CNN architecture} \label{First}            %crea il capitolo
%%%%%%%%%%%%%%%%%%%%%%%%%%%%%%%%%%%%%%%%%imposta l'intestazione di pagina
\lhead[\fancyplain{}{\bfseries\thepage}]{\fancyplain{}{\bfseries\rightmark}}
%\pagenumbering{arabic}  

In this section we introduce one of the main novelty of this paper, a CNN architecture inspired by the structure of the visual system and, in particular, takes into account LGN cells.
The retinal action in a CNN has been implemented in \cite{Lindsey}, where the authors have proposed a bottleneck model for the retinal output. In our model we propose a single filter layer at the beginning of the CNN that should mimic the action of the LGN.
% }
% As far as we know, the action of the LGN has not been implemented in a CNN. 
As we have already discussed in Section \ref{subs:VisualSystem} the RFP of an LGN cell can be modeled by a LoG that acts directly on the visual stimulus. 
% Our aim is to build a CNN architecture that mimics this behavior in order to strengthen the links between CNNs and the visual system.
Since the LGN preprocesses the visual stimulus before it reaches V1, we should add a first layer at the beginning of the CNN that reproduces the role of the LGN. 

% Thus, we introduce a first convolutional layer containing a single filter that eventually should obtain a LoG shape.
 In particular, if we consider a classical CNN we can add before the other convolutional layers, a layer $\ell^0$ composed by only one filter $\Psi^0$ of size $s^0 \times s^0$ and a ReLU function. 
Note that after the first layer $\ell^0$ we will not apply any pooling. In this way taking a classical CNN and adding $\ell^0$ will not modify the structure of the neural network and  the number of parameters will only increase by $s^0 \times s^0$. Furthermore, $\Psi^0$ will prefilter the input image without modifying its dimensions; this behavior mimics the behavior of the LGN which let the neural network to be closer to the visual system structure.
Figure \ref{fig:CNN_LGN_V1} shows a scheme of the first steps of the visual pathway (i.e., LGN and V1) in parallel with the first two layers $\ell^0$ and $\ell^1$ of the LGN-CNN architecture.

\begin{figure}
\centering
\includegraphics[height=9cm]{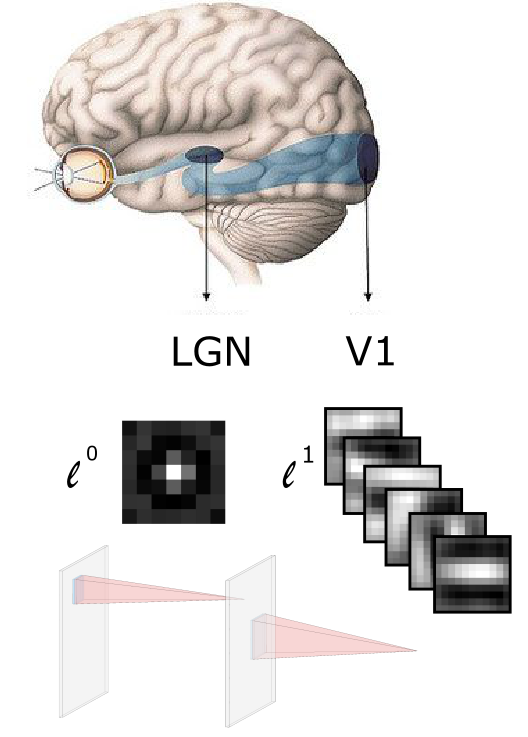}
\caption{Scheme of the LGN and V1 in parallel with the first two layers $\ell^0$ and $\ell^1$ of the LGN-CNN architecture.}
\label{fig:CNN_LGN_V1}
\end{figure}

The theoretical idea behind this structure can be found in a simple result on rotational symmetric convex functionals. 
In particular, we recall that a rotational symmetric convex functional $F$  has a unique minimum $\omega$. Since $F$ is rotational symmetric, $F(\omega \circ g) = F(\omega)$ for a rotation $g$. Thus, since the minimum is unique, $\omega = \omega \circ g$, implying the rotational symmetry of $\omega$.
% In particular, if we have a rotational symmetric convex functional $F$ that has a unique minimum $\omega$ then $\omega$ is also rotational symmetric. Indeed, since $F$ is rotational symmetric, $F(\omega \circ g) = F(\omega)$ for a rotation $g$. Thus, since the minimum is unique, $\omega = \omega \circ g$ and this implies the rotation symmetry of the solution. 
There are several results on symmetries of minimum for functionals as for example in \cite{Lopes}, \cite{Gidas}. Our aim is to extend these results in the case of CNNs in particular on our architecture that we name as Lateral Geniculate Nucleus Convolutional Neural Network (LGN-CNN).

% We have also studied the modifications that occur to the filters in the second convolutional layer. In particular, we have tried to compare their shape with respect to the RFPs of a macaque's V1 from the work of Ringach in \cite{Ringach}. 
% Thus, we have followed the same steps by approximating the filters obtained after the training phase with a Gabor function (\ref{eq:Gabor}) and plotting them on the $(n_x, n_y) $ plane. Indeed, we can compare the elongation in the $x$ and $y$ direction 

We will also show that 
the Gabor-like filters in the second convolution 
layer, reprojected in the 
$(n_x, n_y)$ plane introduced by Ringach and recalled above, 
satisfy the same 
properties of elongation 
which characterizes the RFPs of simple cells in V1. This analysis should enforce the link between our architecture and the visual system structure, at least as regards simple cells in V1.

\section{Applications of LGN-CNN} \label{Second}            %crea il capitolo
%%%%%%%%%%%%%%%%%%%%%%%%%%%%%%%%%%%%%%%%%imposta l'intestazione di pagina
\lhead[\fancyplain{}{\bfseries\thepage}]{\fancyplain{}{\bfseries\rightmark}}
%\pagenumbering{arabic} 

\subsection{Settings}\label{subs:Settings} %\hfill\break

In this Section we  describe the settings for testing our architecture. We use MatLab2019b for academic use.

We train our LGN-CNN architecture on a dataset of natural images called STL-10 (see \cite{STL10}) that contains 5000 training images divided in 10 different classes. 
We have modified the training set in the following way:
%  The steps we have followed for modifying the training set are:
\begin{itemize}
\item Changing the images from RGB color to grayscale color using the built-in function \textit{rgb2gray} of MatLab; 
\item Applying a circular mask to each image, leaving unchanged the image in the circle and putting the external value to zero;
\item Rotating each image by 5 random angles augmenting the dataset to 25000 images; thanks to the previous step no boundary effects are present;
\item Cropping the $64 \times 64 $ centered square that does not contain the black boundaries;
\item Subtraction of the mean value in order to have zero mean input images.
\end{itemize}

Thus, after these steps we have obtained a rotation invariant training set composed by 25000 $64 \times 64$ images. We have applied the same steps to the test set but we have rotated each image to just one random angle. 
Since the images are $64 \times 64$ we have decided to use quite large filters in the first and second layer ($7 \times 7$ and $11 \times 11$ respectively) in order to obtain more information about their shapes.

%we have rotated them randomly by a random angle {\color{green}0, 90, 180 and 270 degrees in order not to loose any information and to prevent any boundary effects due to random angle rotation- togliere e dire che si è preso angoli qualunque}. This was due to the fact that we would like to {\color{black} study a rotation invariant dataset} 
%{\color{green} togliere prevent any possible shift of the center of the filter in the first layer. Indeed, most of images contains the sky which is on the top of the image and could possibly shift the center of the filter towards the bottom. }
%Furthermore we have chosen this dataset since the images' size is $96 \times 96$; in fact we decided to use quite large filters in the first and second layer ($7 \times 7$ and $11 \times 11$ respectively) in order to obtain more information about their shapes.

% \begin{figure*}
% \centering
% \includegraphics[height=15cm, angle=90]{Fig3_old.jpg}
% \caption{Architecture of LGN-CNN.}
% \label{fig:sflcnn}
% \end{figure*}

\begin{figure*}
\centering
\includegraphics[height=5cm]{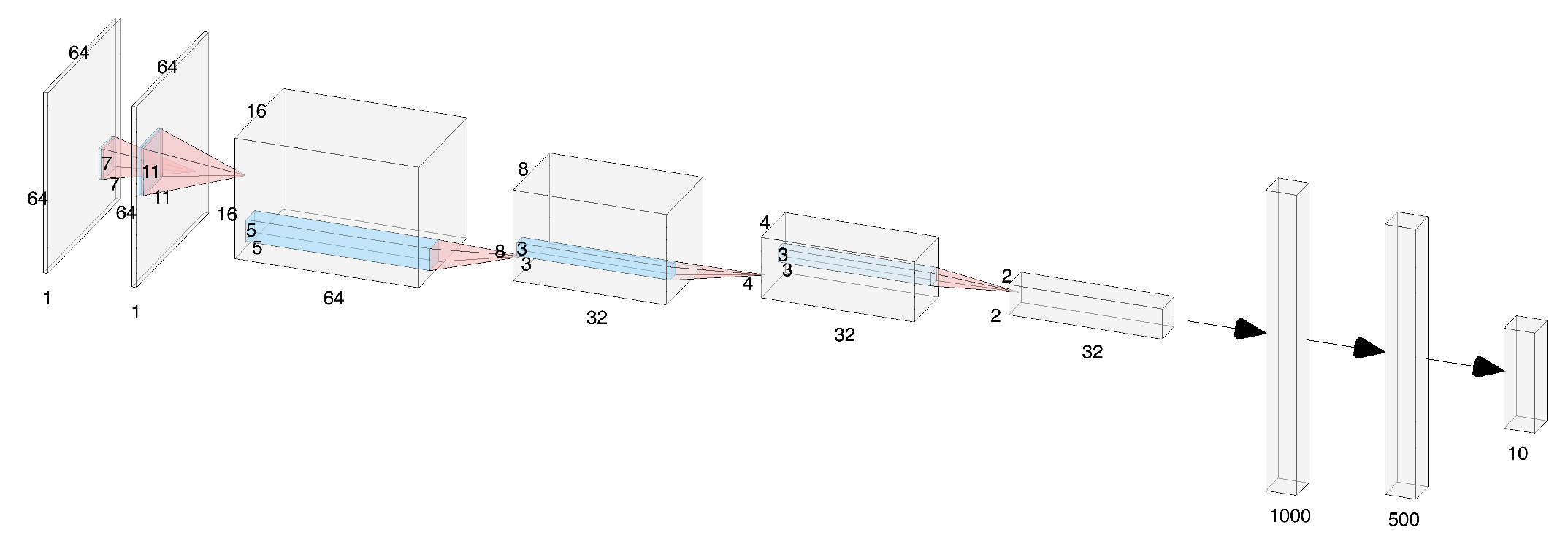}
\caption{Architecture of LGN-CNN.}
\label{fig:sflcnn}
\end{figure*}

Figure \ref{fig:sflcnn} shows the architecture of our CNN. Let us note that between each convolutional layer and its ReLU function there is a batch normalization layer $b$
with the same size of the number of filters of the corresponding convolutional layer. 
The input is an image of size $64 \times 64 \times 1$. Then there is the first convolutional layer $\ell^0$ composed by only $\Psi^0$ of size $7 \times 7$ followed by a ReLU $R$.
After $\ell^0$ the batch normalization layer performs a normalization similar to the one that the retina  performs as described in \cite{Carandini}. Indeed, one of the main difference is the subtraction of the mean value $\mu$ performed by the batch normalization layer defined as follows
$$ \widehat{x_i} = \frac{x_i - \mu}{\sqrt{\sigma^2 + \epsilon}}, $$

where $x_i$ is the element to normalize, $\mu$ is the mean value of the batch, $\sigma$ is the standard deviation of the batch and $\epsilon$ is a small value that prevents bad normalizations in case of really small standard deviations. However, since the input images have zero mean then the convolution with $\Psi^0$ have still zero mean. 
Thus, the batch normalization layer between $\ell^0$ and $\ell^1$ has similar characteristics as the biological one.
On the other hand, the two approaches differ since in the batch normalization layer the statistical parameters ($\mu$, $\sigma$) are calculated channel-wise over all input of the batch, while in the normalization described in \cite{Carandini} the $\sigma$ value is calculated from a single input instance over a restricted spatial neighborhood. However, we expect that the final result will not differ heavily applying the batch normalization layer.

Then, the second layer $\ell^1$ composed by 64 filters of size $11 \times 11$ receives as input a matrix of the same size of the image. Note that the stride is 2 and that the spatial dimensions half. After that %{\color{green}the net applies } 
we apply a ReLU and a max POOLING $p^2_m$ with squares of size $2 \times 2$. The third convolutional layer $\ell^2$ is composed by 32 filters of size $5 \times 5 \times 64$ and it is followed by a ReLU and a max POOLING. Then, we apply a convolutional layer $\ell^3$ composed by 32 filters of size $3 \times 3 \times 32$ followed by a ReLU and a max POOLING. The last convolutional layer $\ell^4$ has the same filters as $\ell^3$ followed by a ReLU and a max POOLING. 
%is composed by 32 filters of size $3 \times 3 \times$ followed by a ReLU and a max POOLING $p_m$ with squares of size $2 \times 2$. 
Eventually three fully-connected ($FC$) layers of size 1000, 500 and 10 respectively are applied giving as output a vector of length 10. Finally,  %{\color{green}the net applies } 
we apply a softmax $\sigma$, 
$\sigma(x)_z = \frac{e^{x_z}}{\sum_k e^{x_k} } $ where $x$ is the output of $FC^3$,
%defined in equation (\ref{eq:softmax}) 
in order to obtain a probability distribution over the 10 classes. The functional that models this neural network is the following 
%\begin{equation}
%F (I, \theta) := \left(\ell^0 \circ R \circ \ell^1 \circ R \circ p^4_m \circ \ell^2 \circ R \circ p^4_a \circ FC^1 \circ FC^2 \circ \sigma \right) (I)
%\label{eq:functional}
%\end{equation}
%\begin{multline} \label{eq:functional}
%F (I) := (  \sigma \circ FC^{3} \circ FC^{2} \circ FC^{1} \circ p^{2}_{m} \circ R \circ b \circ \ell^{4} \circ p^{2}_{m} \circ R \circ b \circ \ell^{3} \circ \\ 
%p^{2}_{m} \circ R \circ b \circ \ell^{2} \circ p^{2}_{m} \circ R \circ b \circ \ell^{1} \circ R \circ b \circ \ell^{0} ) (I)
%\end{multline}
\begin{equation}\label{eq:functional}
\begin{split}
F (I) := & (  \sigma \circ FC^{3} \circ FC^{2} \circ FC^{1} \circ \\ & p^{2}_{m} \circ R \circ b \circ \ell^{4} \circ p^{2}_{m} \circ R \circ \\ & b \circ \ell^{3} \circ 
p^{2}_{m} \circ R  \circ b \circ \ell^{2} \circ \\ & p^{2}_{m} \circ R \circ  b \circ \ell^{1} \circ R \circ b \circ \ell^{0} ) (I)
\end{split}
\end{equation}

%where $\theta$ represents all the parameters of the net present in the convolutional layers and FC layers. 
A cross-entropy loss for softmax function defined as in equation (\ref{eq:loss}) is applied to the functional (\ref{eq:functional}) where $\widetilde{z}$ is the label selected by the neural network and $y(I)$ is the true label.
\begin{equation}\label{eq:loss}
\begin{split}
L (F(I), y(I)) =  & \log ( \sum_z e^{(F_z (I) - F_{\widetilde{z}} (I) )} ) + \\ & F_{\widetilde{z}} (I) - F_{y(I)}
\end{split}
\end{equation}

%\begin{equation}
%L (F(I), y(I)) = F_{\widetilde{z}} (I)+ \log ( \sum_z e^{(F_z (I) - F_{\widetilde{z}} (I) )} ) - F_{y(I)}
%\label{eq:loss}
%\end{equation}

We have trained the neural network for 30 epochs with an initial learning rate of 0.01, a learning rate drop factor of 0.97 and a piecewise learning rate schedule with a learning rate drop period of 1. The mini batch size is 128 with an every-epoch shuffle, the L2 regularization term is 0.02 and the momentum is 0.9.

In Table \ref{Performance_tab} there are summarized the mean performances over 10 different trainings of three CNN architectures with the LGN layer $\ell^0$ plus several convolutional layers (4, 8 and 12 respectively). As expected,  the performance increases when the architecture is deeper. We  stress out that the results obtained on the LGN layer and the first convolutional layer are not affected from the rest of the neural network. Indeed, it is possible to add more layers and build more deeper architectures to obtain better results on the classification task without losing the properties of the first layers of the LGN-CNN. Since our focus was to analyze the structures that arises in the first two layers, we have not further investigated the performance of the CNN.

\begin{table}[]
\centering
\begin{tabular}{|c| c | c |} 
 \hline
 LGN + 4CL & LGN + 8CL &  LGN + 12CL  \\ 
 \hline
 70.41\% & 72.73 \%  & 73.61 \%  \\ 
 \hline

\end{tabular}
\caption{Mean performances of several architectures over 10 different trainings with LGN layer plus 4, 8 and 12 other convolutional layers (CL).  }
\label{Performance_tab}
\end{table}

%{\color{black}We want  to observe that we have modified our net in several ways as for example by changing the training parameters, the number of convolutional layers and filters. The rotational invariance of the first filter and the Gabor shape of filters of the second layer are always present as a result of the architecture of LGN-CNN. In the next Sections we are going to study the filters of the net we have trained.}

%{\color{black}
%\subsection{Classification performances of the net} \label{subs:per} ...
%}

\subsection{The first layer of LGN-CNN} \label{subs:First_layer}

After the training phase we analyze the neural network focusing on the first layer in this Section. Figure \ref{fig:comparison_filter} shows the filter $\Psi^0$ and \ref{fig:comparison_filterB} shows its approximation with minus the LoG.% which is really close to $\Psi^0$.
The two have a high correlation of 95.21\% computed using the built-in function of MatLab \textit{corr2}.
%the comparison with respect to minus the LoG. 
%Let us note that figure \protect\ref{fig:comparison_filterA} is the actual result of the net after training which has positive value in the center and negative ones around it. 
%We can see that figure \ref{fig:comparison_filterB} showing minus the LoG has a shape really close to $\Psi^0$. %with a similar rotationally symmetric pattern. 

 \begin{figure*}
\centering
%\subfloat[][\emph{Filter $\Psi^0$ of first layer of our architecture}.] {\includegraphics[height=5cm]{CNN_1st_filter-2.png}\label{fig:comparison_filterA}} 
%\hspace{3mm}
\subfloat[][\emph{Filter $\Psi^0$ of first layer of LGN-CNN}.]
{\includegraphics[height=5cm]{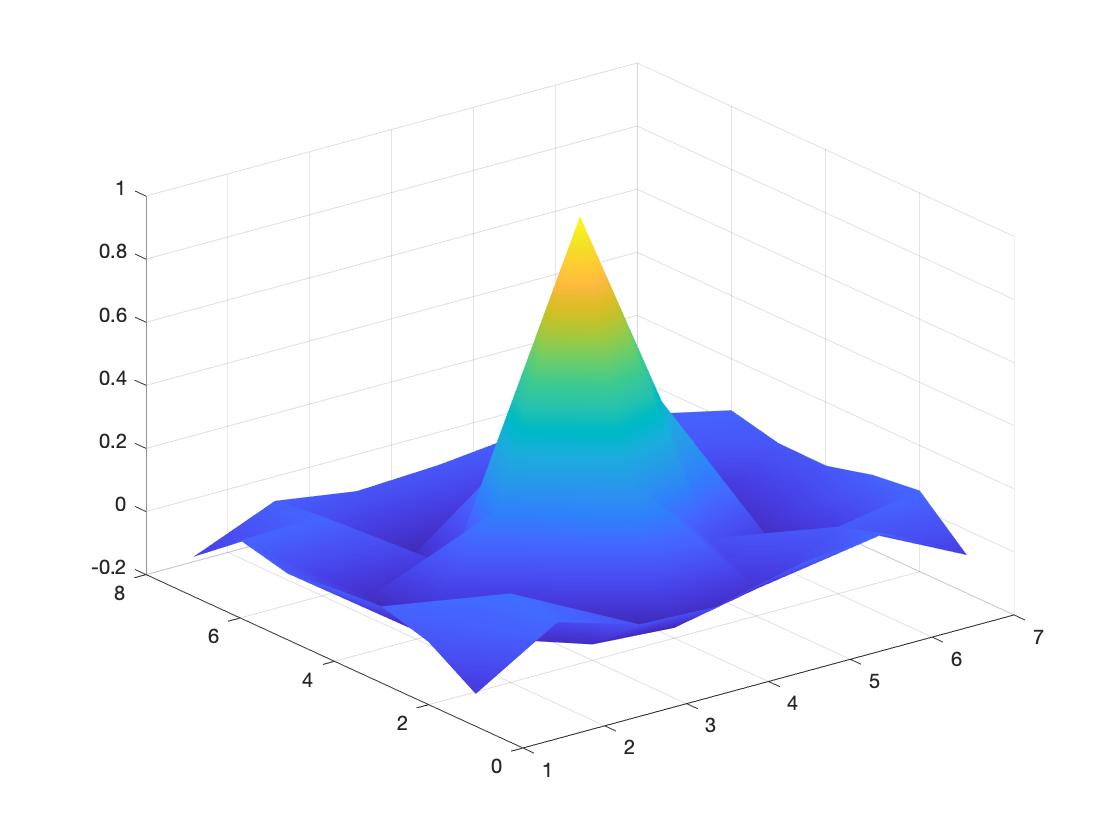}\label{fig:comparison_filterA}}
%\hspace{3mm}
\subfloat[][\emph{Minus the Laplacian of Gaussian (LoG)}.]
{\includegraphics[height=5cm]{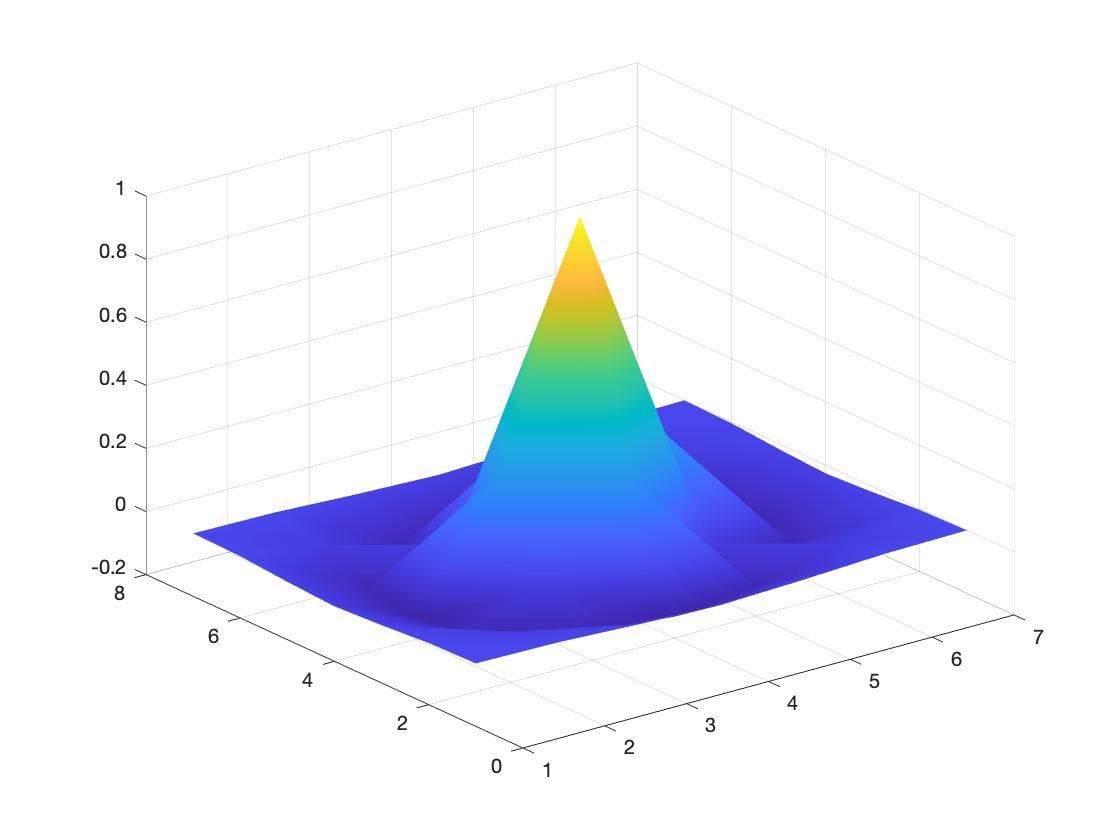}\label{fig:comparison_filterB}}
\caption{Comparison between the filter $\Psi^0$ and minus the LoG. The two have a high correlation of 95.21\% computed using the built-in function of MatLab \textit{corr2}.}
\label{fig:comparison_filter}

\end{figure*}

%% Next figure in the both columns
 \begin{figure*}
\centering
\subfloat[][\emph{Filter $\Psi^0$ of first layer of LGN-CNN}.] {\includegraphics[height=4.5cm]{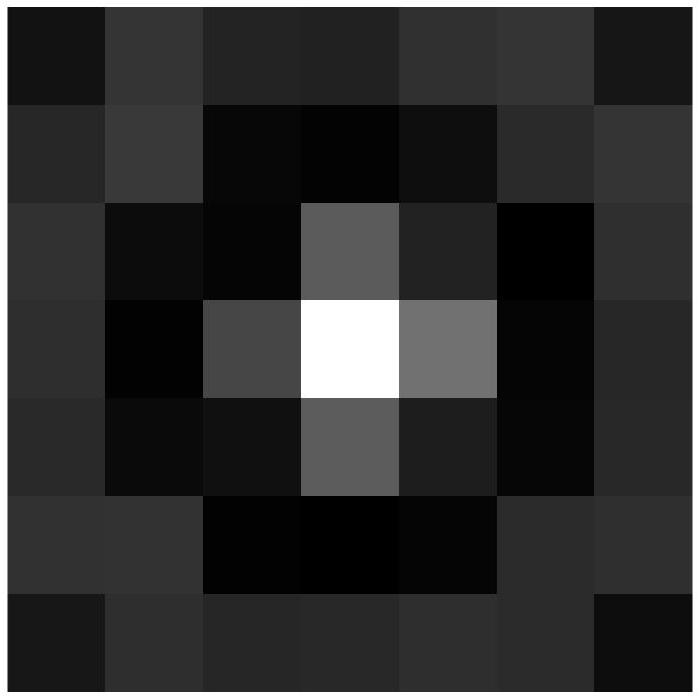}\label{fig:comparison_filterA_2d}} 
\hspace{1mm}
\subfloat[][\emph{Approximation of the filter $\Psi^0$ of first layer of LGN-CNN}.]
{\includegraphics[height=4.5cm]{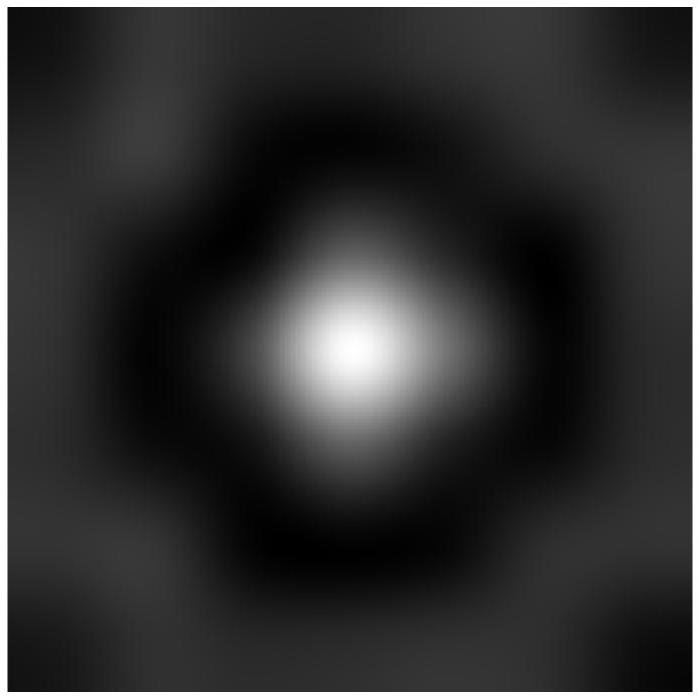}\label{fig:comparison_filterB_2d}}
\hspace{1mm}
\subfloat[][\emph{Minus the Laplacian of Gaussian (LoG)}.]
{\includegraphics[height=4.5cm]{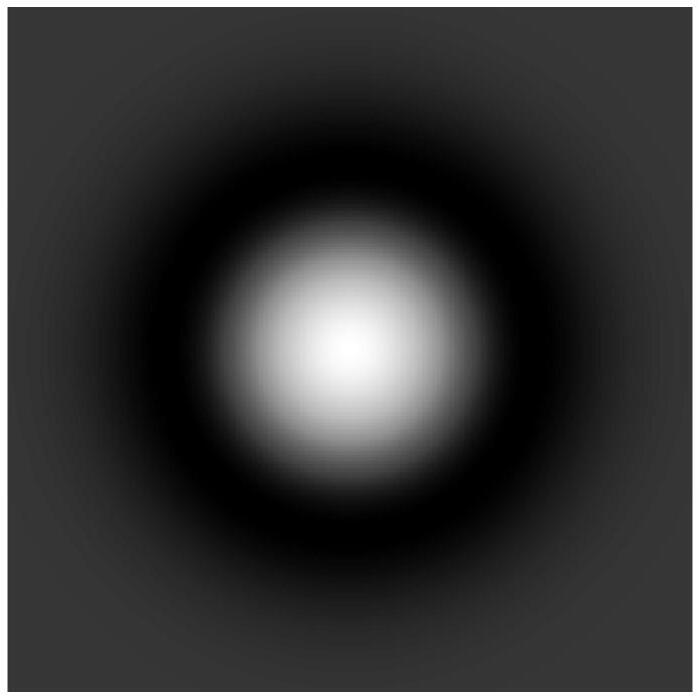}\label{fig:comparison_filterC_2d}}
\caption{The first figure shows the $7 \times 7$ filter $\Psi^0$ of the neural network. To better visualize the filter $\Psi^0$, we provide an approximating $280\times 280$ filter
 and minus the LoG.}% Comparing the last two figures it is possible to see that $\Psi^0$ is close to a LoG.}
\label{fig:comparison_filter_2d}

\end{figure*}

%  \begin{figure}
% \centering
% \subfloat[][\emph{Filter $\Psi^0$ of first layer of LGN-CNN}.] {\includegraphics[height=2cm]{Fig5A.jpg}\label{fig:comparison_filterA_2d}} 
% \hspace{1mm}
% \subfloat[][\emph{Approximation of the filter $\Psi^0$ of first layer of LGN-CNN}.]
% {\includegraphics[height=2cm]{Fig5B.jpg}\label{fig:comparison_filterB_2d}}
% \hspace{1mm}
% \subfloat[][\emph{Minus the Laplacian of Gaussian (LoG)}.]
% {\includegraphics[height=2cm]{Fig5C.jpg}\label{fig:comparison_filterC_2d}}
% \caption{The first figure shows the $7 \times 7$ filter $\Psi^0$ of the neural network. To better visualize the filter $\Psi^0$, we provide an approximating $280\times 280$ filter
%  and minus the LoG.}% Comparing the last two figures it is possible to see that $\Psi^0$ is close to a LoG.}
% \label{fig:comparison_filter_2d}

% \end{figure}

Moreover, Figure \ref{fig:comparison_filterA_2d} shows the 2D image of $\Psi^0$ obtained after the training phase. Then, in Figures \ref{fig:comparison_filterB_2d} and \ref{fig:comparison_filterC_2d} we plot a $2D$ approximation of $\Psi^0$ as a $280\times 280$ filter  
 and minus the LoG in which the rotational symmetric pattern is clearer. 
 
 We would like to point out that in the LGN there exists both cells with on-center/off-surround as well as off-center/on-surround RFPs. In order to model these kinds of cells we have modified the current architecture by adding a second filter in the layer $\ell^0$. Since the STL10 dataset contains natural images with many contours we should expect to have oriented filters. This is not the case and Figure \ref{fig:2filt} shows the on-center/off-surround and off-center/on-surround obtained after the training phase. For simplicity, we have considered the model with a single filter.
 
 \begin{figure}
\centering
\includegraphics[height=3cm]{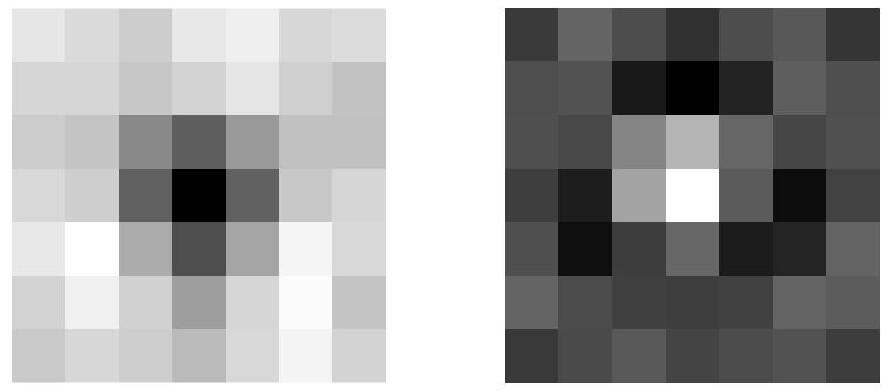}
\caption{On-center/off-surround and off-center/on-surround filters of $\ell^0$ with 2 filters.}
\label{fig:2filt}
\end{figure}
 
%Let us note that the filter attains this shape without any constraint except the structure of the neural network itself. This suggests that choosing the neural network architecture influences directly the filters we will obtain, letting us to link the neural network architecture to the cortical layers structure.
% \newline

In order to quantify the rotational properties of the filter, we compute the correlation between the filter obtained in the first layer of the CNN and a new one, obtained via a rotation invariance symmetrization. We will test the behavior of the first layer of LGN-CNN for different values of the L2 regularization term and adding more convolutional layers.
In particular, we added to the previous architecture described in Section \ref{subs:Settings}  two convolutional layers composed by 32 filters of size
$3\times 3\times 32$ (each one followed by a batch normalization layer b and ReLU R ) after the third layer $\ell^3$.
% and its corresponding batch normalization layer b and ReLU R. 
We have added other  two convolutional layers with the same characteristics after the layer $\ell^4$ .

%In order to study the behavior of $\Psi^0$ we have
%estimated how close the normalized filter is with respect to its corresponding normalized rotational invariant filter calculating their correlation.
% calculated the relative error between the normalized filter itself and its corresponding normalized rotational invariant filter. 
The normalized rotational invariant filter is obtained by the following procedure:
\begin{itemize}
\item Using the function \textit{imresize} with scale 3 and bilinear method to enlarge the filter;
\item Rotating the filter with \textit{imrotate} by 360 discrete angles between 1 and 360 degrees and summing them up;
\item Applying again the function \textit{imresize} with scale 1/3 and nearest method to recover a filter with the same size of $\Psi^0$;
\item Normalizing the filter by subtracting the mean and dividing it by $L^2$ norm.
\end{itemize}

We call the filter obtained in this way $\Psi^0_S$ and we estimate the correlation between $\Psi^0$ and $\Psi^0_S$ using the Matlab function \textit{corr2}. 
%we calculate the relative error as
%\begin{equation}
%\label{eq:Rel_err}
%E_R(\Psi^0) = \frac{|| \Psi^0 - \Psi^0_R ||_F}{||\Psi^0 ||_F}
%\end{equation} 
%where $||\cdot ||_F$ is the Frobenius norm.
%The relative error estimates how far the filter is with respect to a rotationally invariant filter. 

In Table \ref{Stat_tab} we have reported the correlations with different values of L2 regularization term and for two LGN-CNN architectures, the one we have introduced in Section \ref{subs:Settings} which is indicated by 'LGN + 4 layers' and the deeper one described in this Section indicated by 'LGN + 8 layers'. 
All the other training parameters are provided in Section \ref{subs:Settings}. 
%When the correlation is big for a certain couple $\Psi^0$, $\Psi^0_S$ it means that the filter is close to be rotationally symmetric. 
As we can see for 'LGN + 4 layers' architecture the L2 regularization term that let $\Psi^0$ to be the closest rotational symmetric filter is 0.02. This is the filter we have shown in Figures \ref{fig:comparison_filter} and \ref{fig:comparison_filter_2d} and we have used through the paper. 
%On the other hand, in the case of 'LGN + 8 layers' architecture the more rotationally symmetric filter is the one with 0.045 as the L2 regularization term. This behavior suggests that the architecture of LGN-CNN influenced directly the properties of the filters and that 
The Table \ref{Stat_tab} shows that the rotational symmetry of $\Psi^0$ is stable with respect to variations of L2 regularization term and adding convolutional layers.
%just adjusting some parameters could lead to a rotationally symmetric $\Psi^0$. 

\begin{table}[]
\centering
\begin{tabular}{|c| c | c |} 
 \hline
 L2 term & LGN + 4 layers &  LGN + 8 layers  \\ 
 \hline \hline
 0.01 & 88.3 \% & 85.40 \%  \\ 
 \hline
 0.02 & \cellcolor{cyan} 97.15 \% & 90.79 \%  \\ 
 \hline
 0.03 & 93.06 \% & 91.41 \% \\ 
 \hline
 0.04 & 95.61\% & 92.58 \%  \\ 
 \hline
 0.045 & 93.55 \% & \cellcolor{cyan} 94.79 \%  \\ 
 \hline
  0.05 & 95.44 \% & 94.11 \% \\ 
\hline
\end{tabular}
\caption{Correlation between $\Psi^0$ and $\Psi^0_S$ varying the L2 regularization term and the number of layers of the LGN-CNN architecture. Best rotational symmetric filters are selected in cyan for both architectures. }
\label{Stat_tab}
\end{table}

Furthermore, we have studied the properties of $\Psi^0$ varying the data augmentation (DA) applied to the dataset.
Table \ref{Psi_tab} shows the correlation of $\Psi^0$ with the LoG and the correlation of $\Psi^0$ with $\Psi^0_S$ with three different DA applied to the dataset. In the first one we have not applied any DA (No DA); in the second one we have randomly rotated the images of an angle of 0, $\frac{\pi}{2}$, $\pi$, $\frac{3\pi}{2}$ radiant (Mild DA); in the third one we have applied the DA described in Section \ref{subs:Settings} (Hard DA). From Table \ref{Psi_tab} it emerges that the introduction of rotation invariances in the dataset by rotating the images lightly affects the correlation with the LoG but gives stability to $\Psi^0$ allowing it to be  more rotational symmetric.

\begin{table}[]
\centering
\begin{tabular}{| c | c | c | c |} 
 \hline
  & No DA &  Mild DA & Hard DA  \\ 
 \hline \hline
 LoG corr & 93.77 \% & 93.68 \% & 95.21 \% \\ 
 \hline
 $\Psi^0_S$ corr  & 93.92 \% & 94.16 \% & 97.15 \%  \\ 
 \hline
 
\end{tabular}
\caption{Correlation between $\Psi^0$ and its LoG approximation and between $\Psi^0$ and $\Psi^0_S$ varying the data augmentation (DA) applied to the dataset.}
\label{Psi_tab}
\end{table}

In conclusion, thanks to the analysis performed on the properties of $\Psi^0$, we can argue that the structure of the architecture itself influences the shape of the filters and that the introduction of $\ell^0$ with a single filter $\Psi^0$ is a good model of the LGN.

%We will deeply investigate the rotation symmetric property of $\Psi^0$ in the section \ref{Third}.

%\subsection{The second layer of the net}
%
% \hfill\break
% 
% It often happens that the first convolutional layer of a CNN contains a bank of Gabor filters. Since LGN-CNN contains a first layer with a single filter that attains a LoG shape, we are interested to analyze what happens to the second layer and see if the filters have a Gabor shape. 

% \hfill\break

\subsection{The second layer of LGN-CNN}\label{subs:RingachStudy}

\begin{figure*}[ht] 
\centering
\includegraphics[height=6cm]{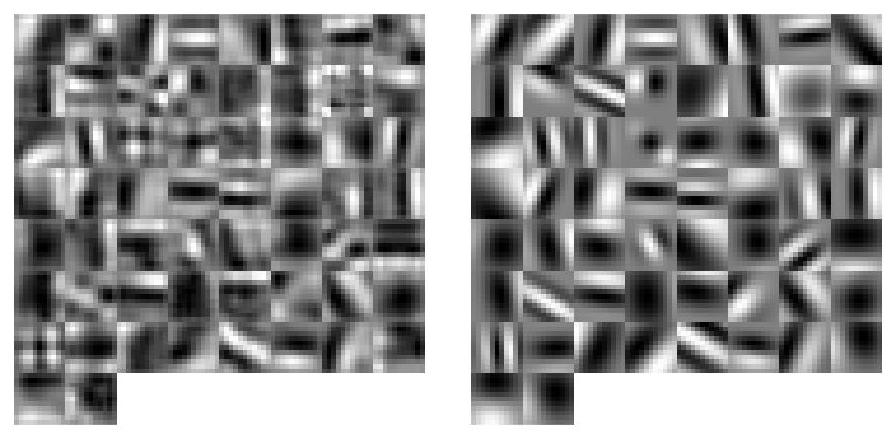}
\caption{On the left: filters from LGN-CNN. On the right: their approximation with the function (\ref{eq:Gabor}).}
\label{fig:filters_comp}
\end{figure*}

\begin{figure*}[ht]
\centering
\includegraphics[height=6cm]{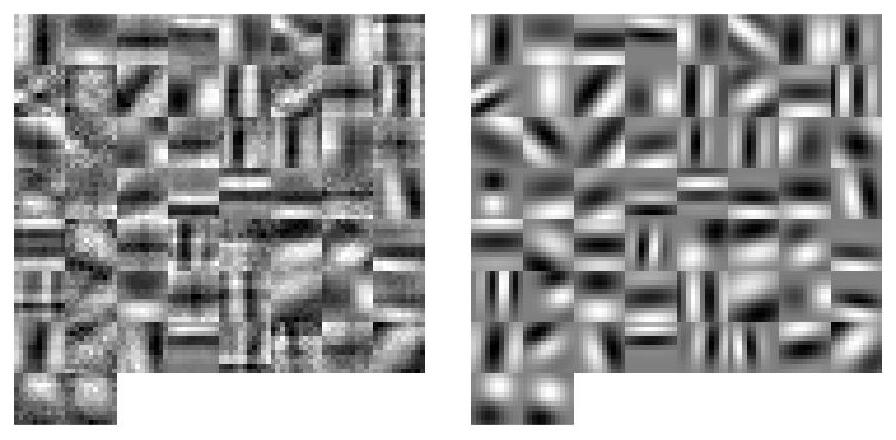}
\caption{On the left: filters from classical CNN. On the right: their approximation with the function (\ref{eq:Gabor}).}
\label{fig:filters_comp74}
\end{figure*}

To enforce the link between our architecture and the structure of the visual system, we have studied the filters in the second layer comparing them with some real data obtained on monkey in \cite{Ringach}. 
Therefore, we have trained two different CNNs, an LGN-CNN defined by the functional (\ref{eq:functional_LGNCNN_small}) %\qq{Aggiunto il funzionale che abbiamo utilizzato per il confronto con Ringach (non era quello di eq. 3)}
\begin{equation}\label{eq:functional_LGNCNN_small}
\begin{split}
F (I) := & (  \sigma \circ FC^1 \circ R \circ \ell^3 \circ p^4_a \circ R \circ \ell^2 \\ & \circ p^4_m \circ R \circ \ell^1 \circ R \circ \ell^0 ) (I)
\end{split}
\end{equation}

and a classical CNN defined by the functional (\ref{eq:functional_classic}) in which we have eliminated the first convolutional layer $\ell^0$ and its following ReLU $R$, characteristic of our architecture.
\begin{equation}\label{eq:functional_classic}
\begin{split}
F (I) := & (  \sigma \circ FC^1 \circ R \circ \ell^3 \circ p^4_a \circ R \circ \ell^2 \\ & \circ p^4_m \circ R \circ \ell^1  ) (I)
\end{split}
\end{equation}

%\begin{equation}
%F (I, \theta) := \left(\ell^1 \circ R \circ p^4_m \circ \ell^2 \circ R \circ p^4_a \circ FC^1 \circ FC^2 \circ \sigma \right) (I)
%\label{eq:functional_classic}
%\end{equation}

Let us note that in both architectures $\ell^1$ contains filters with Gabor shapes after training. This is a well-known result on the filters of the first convolutional layer of CNNs as for example in \cite{Poggio2}, \cite{Yamins}; however, the introduction of a first layer composed by a single filter does not change this behavior.
% , enforcing the link of our architecture and the visual system structure. 
Indeed, we have studied the statistical distribution of these banks of filters confronting the results with the real data of Ringach. 

In the case of LGN-CNN we have not approximated the filters in $\ell^1$ directly but the filters obtained by the convolution with $\Psi^0$.

We have approximated the filters in the banks using the function (\ref{eq:Gabor}); Figure \ref{fig:filters_comp} shows some of the filters of LGN-CNN and their approximation and the same occurs in Figure \ref{fig:filters_comp74} in the case of classical CNN. 
Let us note that the mean correlation estimated with the built-in MatLab function \textit{corr2} increases from classical CNN to LGN-CNN from just 71.62\% to 93.50\%. This suggests that introducing the layer $\ell^0$ with a single filter better regularize the filters in the following convolutional layer $\ell^1$.
%% Prima era scritto così 
%Let us note that the mean relative error decreases from classical CNN to LGN-CNN from 60.13 \% to 44.69 \%. This suggests a better regularization of the filters in our architecture that can be seen directly in Figures \ref{fig:filters_comp} and \ref{fig:filters_comp74} where the filters in classical CNN are noisier. Figure \ref{fig:plot_err_approx_57_77} shows the relative errors that occurs by approximating the filters of the second layer of LGN-CNN by a Gabor function compared to the approximation of filters of a classical CNN. They are plotted in increasing order w.r.t. the relative errors attained. It is clear that the relative error of filters in classical CNN is greater than the relative error in the case of LGN-CNN; this suggests that introducing a new convolutional layer with a single filter better regularize the filters in the following layer.

%{\color{green} il significato di $(n_x, n_y)$ va messo la prima volta che si introducono. Perché fa parte del modello di Ringach non del nostro}
%We can define $(n_x, n_y) = (\sigma_x \cdot f, \sigma_y \cdot f)$ where $n_x$ and $n_y$ estimate the elongation in $x$ and $y$ directions respectively thanks to $\sigma_x$ and $\sigma_y$. They are rescaled by $f$ which indicates how far the shape of the filter is with respect to a Gaussian; in particular, if $f=0$ the function $h$ in (\ref{eq:Gabor}) simplifies to a Gaussian since the cosine becomes a constant. 

 \begin{figure*}[ht]
\centering
\subfloat[][\emph{Distribution of filters in $\ell^1$ \\ of  classical CNN defined by \\ functional (\ref{eq:functional_classic})}.] 
{\includegraphics[height=5cm]{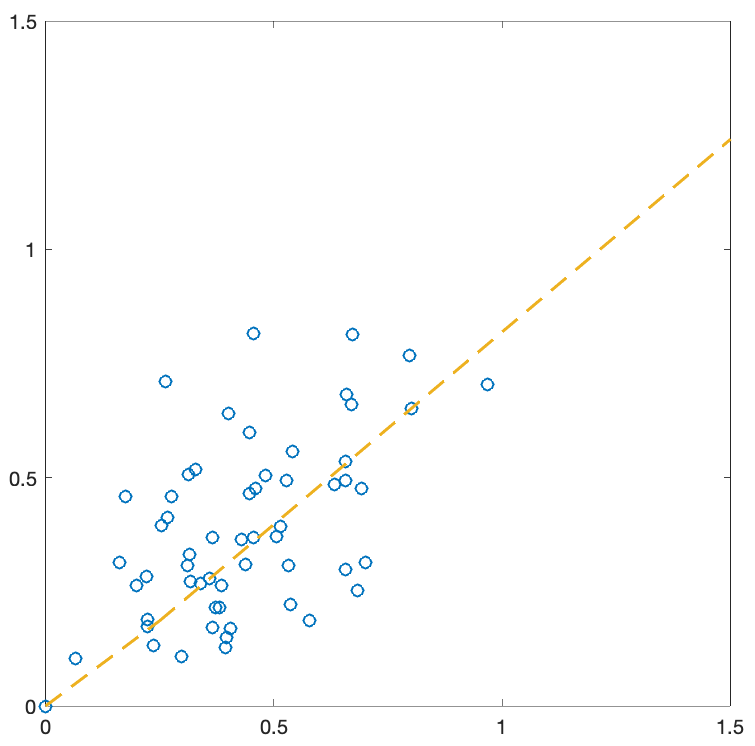}\label{fig:stat_distrA}} 
%\hspace{0.1mm}
\subfloat[][\emph{Distribution of filters in $\ell^1$ \\ of LGN-CNN defined by  \\ functional (\ref{eq:functional_LGNCNN_small})}.]
{\includegraphics[height=5cm]{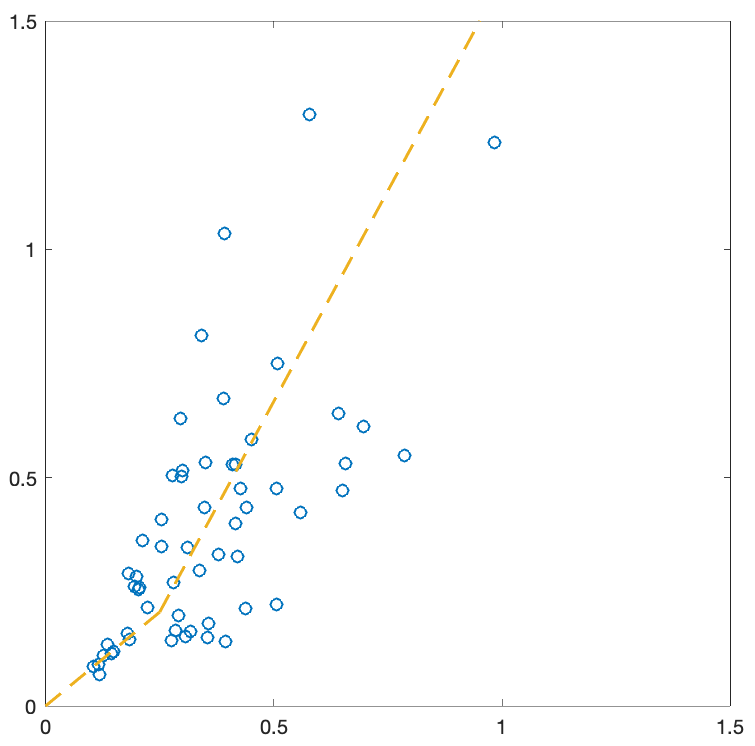}\label{fig:stat_distrB}}
%\hspace{0.1mm}
\subfloat[][\emph{Distribution of RFPs of \\ simple cells from \cite{Ringach}}.]
{\includegraphics[height=5cm]{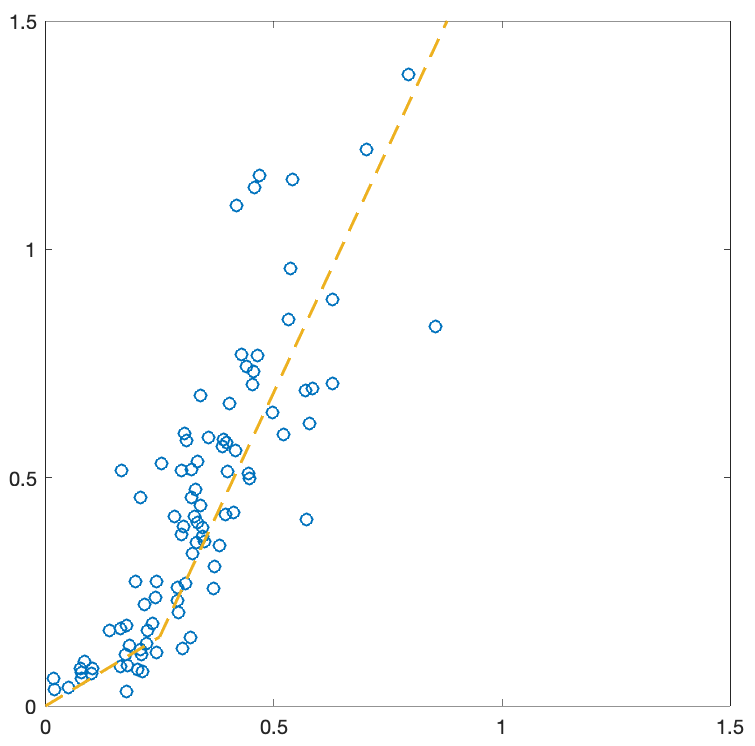}\label{fig:stat_distrC}}
\caption{Comparison between the statistical distribution on $(n_x, n_y)$ plane of filters of a classical CNN, of our architecture and of RFPs of real data.}
\label{fig:stat_distr}

\end{figure*}

We follow the same step as Ringach in \cite{Ringach} by plotting in the $(n_x, n_y)$ plane. In order to compare the plots, we looked for the distribution that best fits the neural data. In particular, it approximates the points closer to the origin with a line $y=\alpha x$ and then it approximates the rest of the points with a line starting from the end of the previous one.

%In order to better compare the plots we looked for the line $y=\alpha x$ that best fits the data.

% \begin{figure}
% \centering
% \includegraphics[height=5.5cm]{Fig8.jpg}
% \caption{Relative errors referred to Gabor function approximations of filters of LGN-CNN in Figure \ref{fig:filters_comp} and of classical CNN in Figure \ref{fig:filters_comp74}. On the $x$-axis there are the filters ordered w.r.t. their approximation relative errors and on the $y$-axis the relative error itself.}
% \label{fig:plot_err_approx_57_77}
% \end{figure}

%(\ref{eq:Gabor})

Figure \ref{fig:stat_distr} shows the three plots. Let us note that introducing $\ell^0$ modifies the elongation of Gabor filters in $\ell^1$. In particular, in classical CNN the filters are often more elongated in the $x$ direction as we can see from the slope of the interpolating line in Figure \ref{fig:stat_distrA}. In Figure \ref{fig:stat_distrB} we can see that the slope changes greatly and that the filters become much more elongated in the $y$ direction. This behavior is the same in the case of RFPs (Figure \ref{fig:stat_distrC}) in which the distribution has a similar slope of LGN-CNN. This enforces more the link of LGN-CNN with the structure of the visual system motivating us to pursue in this direction.

\section{Retinex algorithm via learned kernels} \label{subs:Gen_appr_Retinex}            %crea il capitolo
%%%%%%%%%%%%%%%%%%%%%%%%%%%%%%%%%%%%%%%%%imposta l'intestazione di pagina
\lhead[\fancyplain{}{\bfseries\thepage}]{\fancyplain{}{\bfseries\rightmark}}
%\pagenumbering{arabic} 

In this section we  test the rotational symmetric filter on Retinex effect and contrast-based illusions. 
We described in Section \ref{subs:VisualSystem} how the model of 
\cite{Morel1} have been neurally interpreted in \cite{Gauge} and applied to LoG. Our approach 
aims to find the reconstructed image $\widetilde{I}$
for a general operator 
\begin{equation}
\label{eq:M}
\;\;\;\;\; \;\;\;\;\; \;\;\;\;\; \;\;\;\;\;  M : \Xi \subset\R^2 \to \R,
\end{equation}
%The same procedure can be applied to our approximated LoG. 
solving the following problem 
% and in particular to our approximated LoG.}
% Thus, we need to find a function $\widetilde{I} $ such that 
\begin{equation}
\label{eq:PDE_Itilde}
\;\;\;\;\; \;\;\;\;\; \;\;\;\;\; \;\;\;\;\; M \ast \widetilde{I} = M \ast I.
\end{equation}

by finding the inverse operator of $M$ 
\begin{equation}
\label{eq:Mtilde}
\;\;\;\;\; \;\;\;\;\; \;\;\;\;\; \;\;\;\;\;  \widetilde{M}  : \Xi \subset\R^2 \to \R.
\end{equation}

% and, in particular, in the case of filters obtained after a training phase in a CNN. We would like to test this approach to the filter $\Psi^0$ of the LGN-CNN we have described in Section \ref{subs:First_layer} and compare it with a LoG.

% This simply means that we need to compute the inverse $\tilde M$ of the operator $M$ associated to the filter $\Psi^0$ of the LGN-CNN.

%that acts similarly as the LGN in the visual system by prefiltering the images with a LoG shape convolution (for reference see \cite{Bertoni}). 

By definition the  inverse $\widetilde{M}$ of the operator $M$ satisfies 
\begin{equation}
\label{eq:PDE_Mtilde}
\;\;\;\;\; \;\;\;\;\; \;\;\;\;\; \;\;\;\;\; \;\;\; M \ast \widetilde{M} = \delta.
\end{equation}

Applying the steepest descent method, we obtain an iterative process, which can be formally expressed as 
\begin{equation}
\label{eq:Mtilde_discrete}
\;\;\;\;\;\;\;\; \widetilde{M}_{t+1} = \widetilde{M}_{t} + dt \cdot (M \ast \widetilde{M}_{t} - \delta ). 
\end{equation}

The algorithm will stop at time $T$ when $\frac{|| \widetilde{M}_{T+1} - \widetilde{M}_{T}||_{L^1}}{dt} < \epsilon$, for a fixed error $\epsilon >0$. Thus, $|| M \ast \widetilde{M}_{T} - \delta  ||_{L^1} < \epsilon$ and indeed $\widetilde{M}_{T}$ would be a good approximation of $\widetilde{M}$.
Finally, we can compare the 
image $I$ with the reconstructed one $\widetilde{I}$ and see if any Retinex effects occur. 
Let us note that one  difference between our approach and the one proposed by Morel is that he imposed the  Neumann boundary conditions to the Poisson equation whereas, in our model, the Neumann boundary conditions imposed to the filter are inherited from  the inverse operator itself. Furthermore, we have faced the problem for a Laplacian operator to see if the results are similar to the Morel ones.
% Furthermore, we have faced the problem for a convolutional operator $M$ that approximates the Laplacian operator of Morel's approach.
%Morel has faced the problem for the Laplacian operator whereas we have faced it for a convolutional operator $M$ that approximates. 
Indeed, in our Retinex algorithm we convolve a given visual stimulus $I$ with a fundamental solution $M$ and we reconstruct the perceived image $\widetilde{I}$ using the inverse operator $\widetilde{M}$.

\subsection{Study of the inverse operators} \label{sec:Inverse}            
%\hfill\break

Firstly, we show the inverse operators obtained through the algorithm described by equation (\ref{eq:Mtilde_discrete}). We start from the classical discrete Laplacian operator and then we will move to  convolutional operators, in particular a discrete LoG and the filter $\Psi^0$ of an LGN-CNN. In order to compare the inverse operators, we will show the 2D plots obtained by selecting a slice of the 3D inverse operator itself.

We  start by comparing the inverse of the discrete Laplacian with respect to its exact inverse operator given by the function $\log ( \sqrt{x^2 + y^2})$. Figure \ref{fig:Inverse_plot_Lapl} shows the approximation of the inverse of the Laplacian and  
\ref{fig:Inverse_plot_LaplLog} shows the exact inverse in the interval $[-5,5]$. We can note that 
the approximation is close to the exact one.

%In \cite{Bertoni} the authors have shown that the first filter $\Psi^0$ of the first layer of a LGN-CNN approximates a LoG. 
In the second row of Figure \ref{fig:comparison_Psi0_LoG}  we have compared the inverse operators of a LoG and $\Psi^0$ where Figure \ref{fig:Inverse_plot_LaplGauss_7x7_0p01} shows the LoG inverse operator and Figure \ref{fig:Inverse_plot_1stlayer_CNN_7x7} shows the $\Psi^0$ inverse operator. Since their inverse operators are close we expect similar Retinex effects in the next Section.
% They have almost the same inverse operator and we will see in the next Section that their Retinex effects are the same. 
%This similar behavior suggests that a strong link exists between $\Psi^0$ and a LoG

 \begin{figure*}[ht]
\centering
%\subfloat[][\emph{Filter $\Psi^0$ of first layer of our architecture}.] {\includegraphics[height=5cm]{CNN_1st_filter-2.png}\label{fig:comparison_filterA}} 
%\hspace{3mm}
\subfloat[][\emph{Exact inverse of a Laplacian $\log (\sqrt{x^2 + y^2)}$ in the interval $[-5,5]$}.]
{\includegraphics[height=5cm]{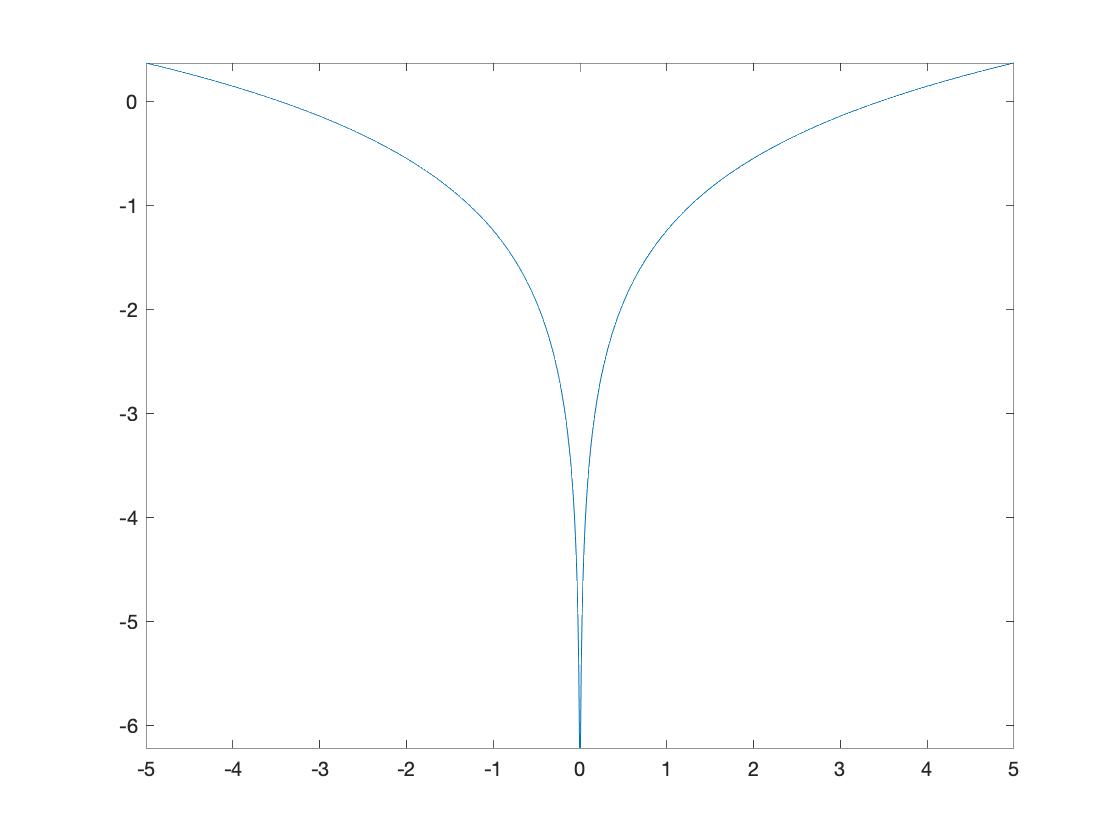}\label{fig:Inverse_plot_LaplLog}}
%\hspace{3mm}
\subfloat[][\emph{Inverse of discrete Laplacian}.]
{\includegraphics[height=5cm]{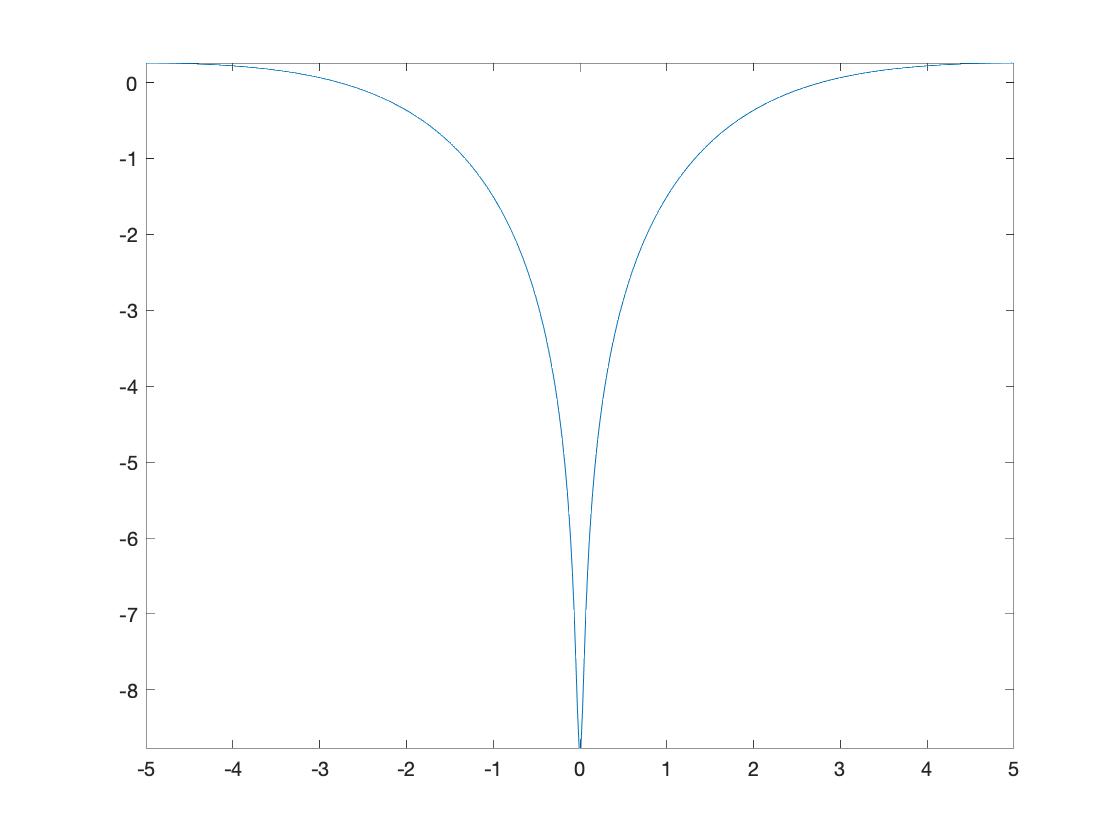}\label{fig:Inverse_plot_Lapl}}
% \caption{Comparison between the inverse of the discrete Laplacian and its exact inverse operator.}
% \label{fig:comparison_inverse_Lapl}

% \end{figure*}

%  \begin{figure*}[ht]
% \centering
%\subfloat[][\emph{Filter $\Psi^0$ of first layer of our architecture}.] {\includegraphics[height=5cm]{CNN_1st_filter-2.png}\label{fig:comparison_filterA}} 
%\hspace{3mm}
%\hspace{3mm}
\subfloat[][\emph{Inverse of LoG}.]
{\includegraphics[height=5cm]{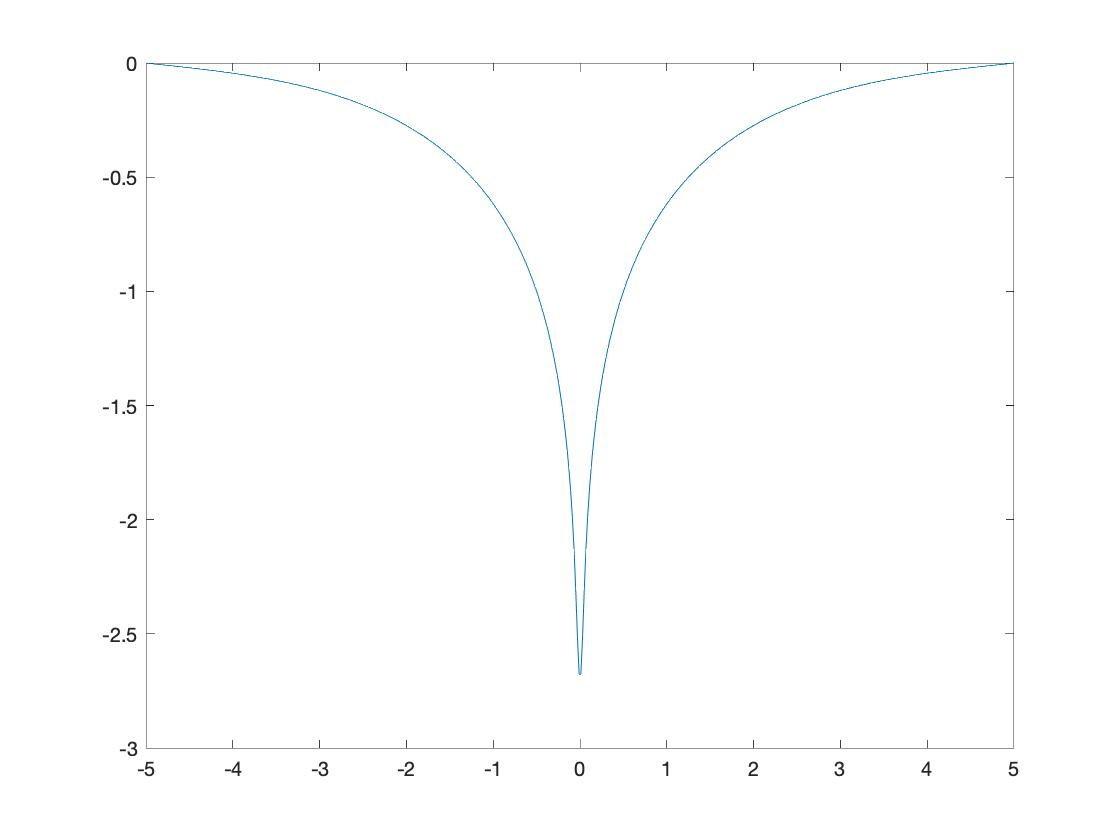}\label{fig:Inverse_plot_LaplGauss_7x7_0p01}}
\subfloat[][\emph{Inverse of the first layer of an LGN-CNN}.]
{\includegraphics[height=5cm]{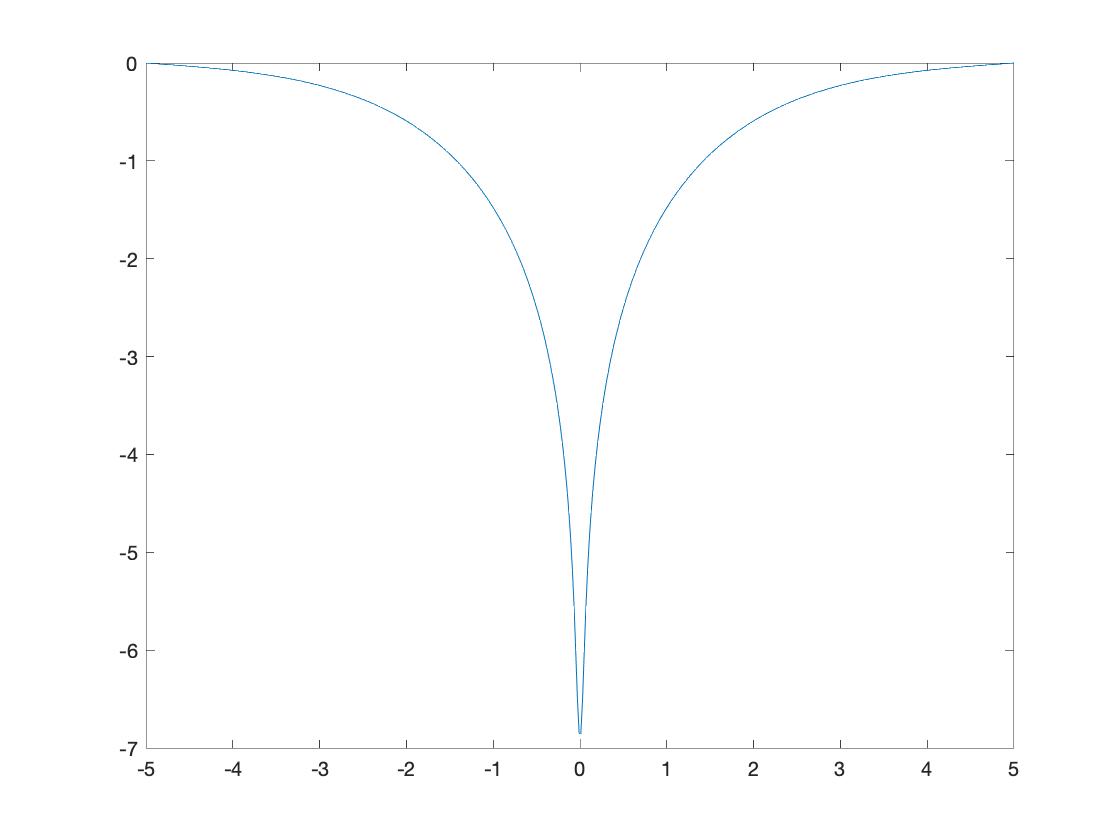}\label{fig:Inverse_plot_1stlayer_CNN_7x7}}
\caption{Comparison between inverse operators. First row: inverse of the discrete Laplacian and its exact inverse operator. Second row: inverse operators of a discrete LoG and the first filter $\Psi^0$ of an LGN-CNN.}
\label{fig:comparison_Psi0_LoG}

\end{figure*}

\subsection{Application of the algorithm} \label{subs:Appl_alg}            
 %\hfill\break

% In Section \ref{subs:Gen_appr_Retinex} we have introduced our model and in Section \ref{sec:Inverse} we have shown the inverse operators for some operators $M$. 
In this Section, we aim to test our algorithm and our different operators to see if Retinex effects occur. Let us note that Retinex is an algorithm that mimics our color perception and does not try to improve the image quality. Indeed, the grayscale values should modify towards our color perception intensity values.

\subsubsection{Circles on a gradient background} %\hfill\break

 \begin{figure*}[ht]
\centering
%\subfloat[][\emph{Filter $\Psi^0$ of first layer of our architecture}.] {\includegraphics[height=5cm]{CNN_1st_filter-2.png}\label{fig:comparison_filterA}} 
%\hspace{3mm}
%\hspace{3mm}

\subfloat[][\emph{Starting image with two gray dots on a gradient background}.]
{\includegraphics[height=3cm]{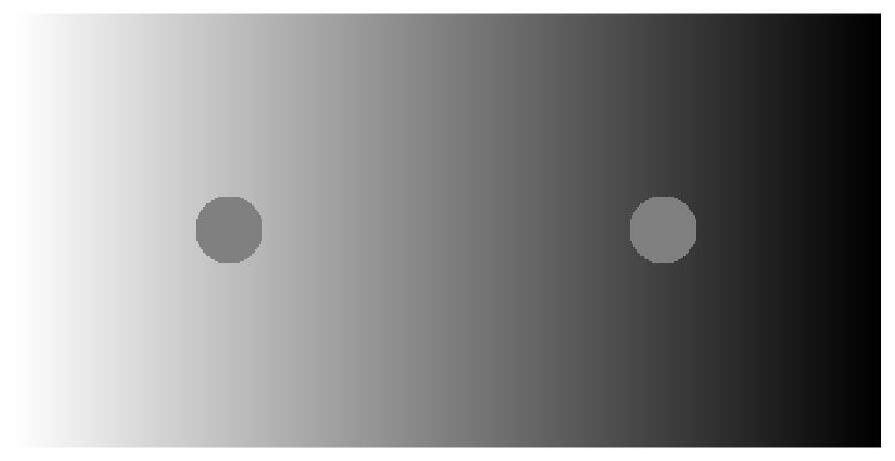}\label{fig:Im_2cerchi}}

%\hspace{3mm}
\subfloat[][\emph{Retinex effects of exact inverse of Laplacian}.]
{\includegraphics[height=3cm]{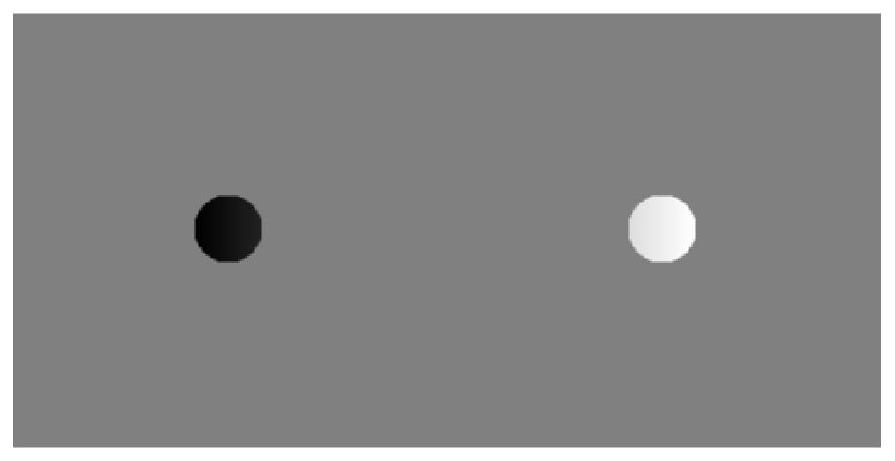}\label{fig:Im_Lapl_esatto}}
\subfloat[][\emph{Retinex effects of discrete Laplacian}.]
{\includegraphics[height=3cm]{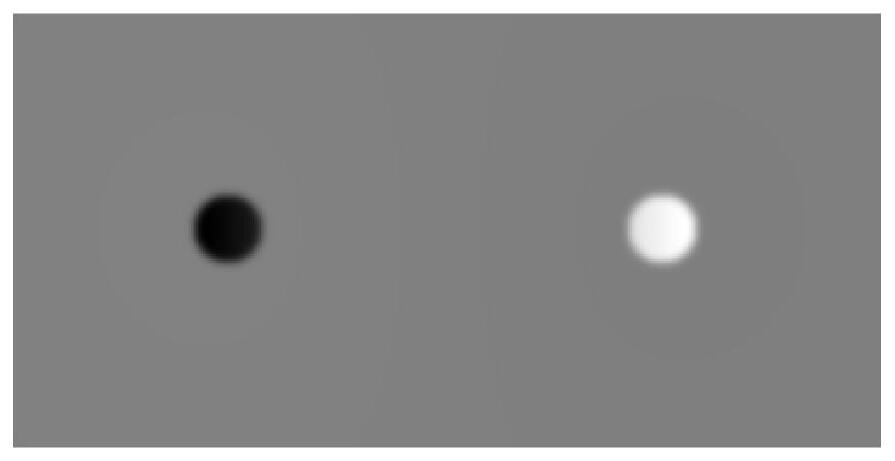}\label{fig:Im_Lapl_discreto}}

\subfloat[][\emph{Retinex effects of LoG}.]
{\includegraphics[height=3cm]{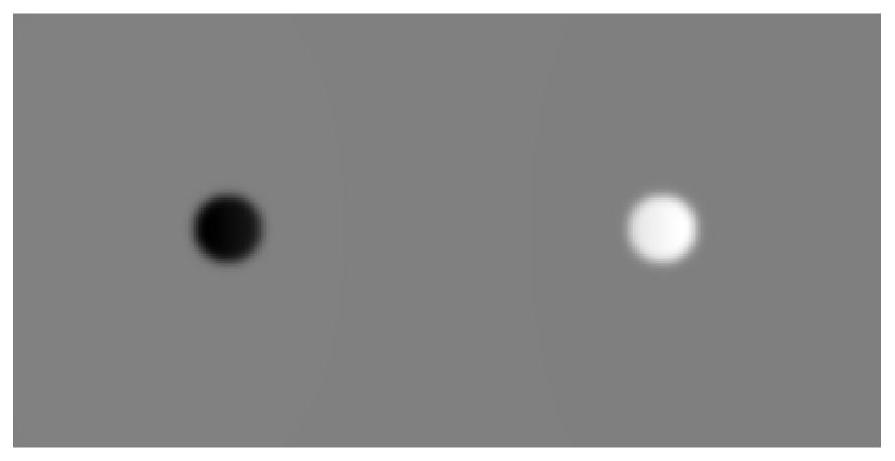}\label{fig:Im_LoG}}
\subfloat[][\emph{Retinex effects of $\Psi^0$ of LGN-CNN}.]
{\includegraphics[height=3cm]{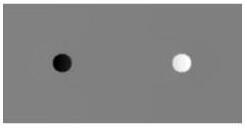}\label{fig:Im_Psi}}
\caption{Retinex effects of some inverse operators on starting image \ref{fig:Im_2cerchi}.}
\label{fig:Comparison_Circles_on_Gradient}

\end{figure*}

We start with a simple grayscale image (see Figure \ref{fig:Im_2cerchi}) in order to see how the different operators work. We are using the same image as in \cite{Morel1} and \cite{Morel2} in which they obtain remarkable Retinex effects. It has a background that shifts from white (value 1) to black (value -1) maintaining the gradient constant from left to right. There are also two gray dots (same value 0), the left one with a brighter background, the right one with a darker one. Because of the different backgrounds, our visual system perceives the two dots differently: the left one is perceived as darker with respect to its true color; the right one is perceived as brighter. 
% \hfill\break

We first consider the discrete Laplacian operator since we expect to obtain the same Retinex effects as in \cite{Morel1}. Figure \ref{fig:Im_Lapl_discreto} shows the result obtained with our method. It is clearly the same result of the experiments performed by Morel and the values obtained in the position of the dots are close to -0.9 and 0.9 (indeed really close to completely black and white dots) whereas the entire background is gray with 0 values.

Thus, we have considered the exact inverse of the Laplacian operator $\log (\sqrt{x^2 + y^2)}$. From Figure \ref{fig:Im_Lapl_esatto} it is clear that the Retinex effects occur to the gray dots. Furthermore we can note that the contours of the two dots are more defined w.r.t. the discrete Laplacian. Also in this case the values in the position of the dots are close to -0.9 and 0.9.
% \hfill\break

Then, we have performed our experiment with the inverse of a LoG. We can note from  Figure \ref{fig:Im_LoG} that the Retinex effects occur also in this case (with values close to -0.9 and 0.9), similarly to the Retinex effects of the discrete Laplacian. Also in this case the contours of the dots are not completely clear.

Finally, we have tested the inverse of the convolutional operator $\Psi^0$ of an LGN-CNN introduced in Section \ref{Second}. In Figure \ref{fig:Im_Psi} we can see that this operator shows Retinex effects where the values of the dots are close to -0.9 and 0.9. In this case the contours of the dots are even clearer than the ones obtained with the LoG and the discrete Laplacian.
%In this case some effects occur around the dots showing a brighter contour for the left one and a darker contour for the right one.
% \hfill\break

To summarize, we have shown that our method reproduces the same results of \cite{Morel1} in the case of the Laplacian operator. Furthermore, we have tested it on other operators obtaining remarkable results. It is particularly interested the case of the convolutional operator $\Psi^0$ since it is able to show Retinex effects even if it is a learned filter with no a-priori structure, enforcing the link between the LGN-CNN architecture and the visual system.

%\begin{figure}[H]
%\centering
% \includegraphics[width=\linewidth]{Comparison_Circles_on_Gradient_2_crop.jpg}
%  \caption{Top left: circles on a gradient. Top center: reconstructed image with the discrete Laplacian operator. Top right: reconstructed image with the exact inverse $\log (\sqrt{x^2 + y^2)}$. Down left: reconstructed image with the LoG. Down right: reconstructed image with the discrete operator $\Psi^0$ of a LGN-CNN.}
%  \label{fig:Comparison_Circles_on_Gradient}
%\end{figure}

\subsubsection{Adelson's checker}
% \hfill\break

We have also tested our algorithm on the Adelson's checker shadow illusion as in \cite{Morel1} and \cite{Morel2}. Since we are more interested in the Retinex effects of LoG and $\Psi^0$ operators we have analyzed their abilities on the grayscale image (see Figure \ref{fig:Im_cb}). 
It shows a checkerboard with light gray and dark gray square with a cylinder on it that shadows a part of the squares. In particular, the square labeled 'A' and the square labeled 'B' have the same grayscale value (in our case, since -1 is black and 1 is white, they have -0,4953 value). The illusion is built in such a way that, even if they have the same value, they are perceived in a completely different way. Indeed square 'A', which is outside the shadow and surrounded by light gray squares, is perceived as a dark gray square. On the other hand, square 'B', which is inside the shadow and is surrounded by dark gray squares, is perceived as lighter.

We expect that the two operators should reproduce the same behavior of our perception, in particular square 'A' should have a smaller value whereas square 'B' should have a bigger value. Figure \ref{fig:Im_LoG_cb} shows the Retinex effects obtained with the LoG operator. In particular, the value of square 'A' changes to -0,6553 whereas the value of square 'B' changes to 0,02351.
Thus, we have studied the behavior of $\Psi^0$ operator whose results are shown in Figure \ref{fig:Im_Psi_cb}. Even in this case Retinex effects occur where the value of square 'A' changes to -0,6035 and the value of 'B' changes to 0,2348.

Figures \ref{fig:Im_LoG_cb_crop}, \ref{fig:Im_LoG_cb_crop} and \ref{fig:Im_Psi_cb_crop} highlight the two squares 'A' and 'B' in the starting image and in the recovered images using the LoG and $\Psi^0$. In this way it is  clearer that the Retinex effects occur in both cases.

To summarize, we have shown that the filter $\Psi^0$ considered as a convolution operator shows Retinex effects really closed to Retinex effects of LoG. This enforces again the link between the structure of our architecture and the structure of LGN.

 \begin{figure*}
\centering
%\subfloat[][\emph{Filter $\Psi^0$ of first layer of our architecture}.] {\includegraphics[height=5cm]{CNN_1st_filter-2.png}\label{fig:comparison_filterA}} 
%\hspace{3mm}
%\hspace{3mm}

\subfloat[][\emph{Grayscale Adelson's checker \\ shadow illusion}.]
{\includegraphics[height=5cm]{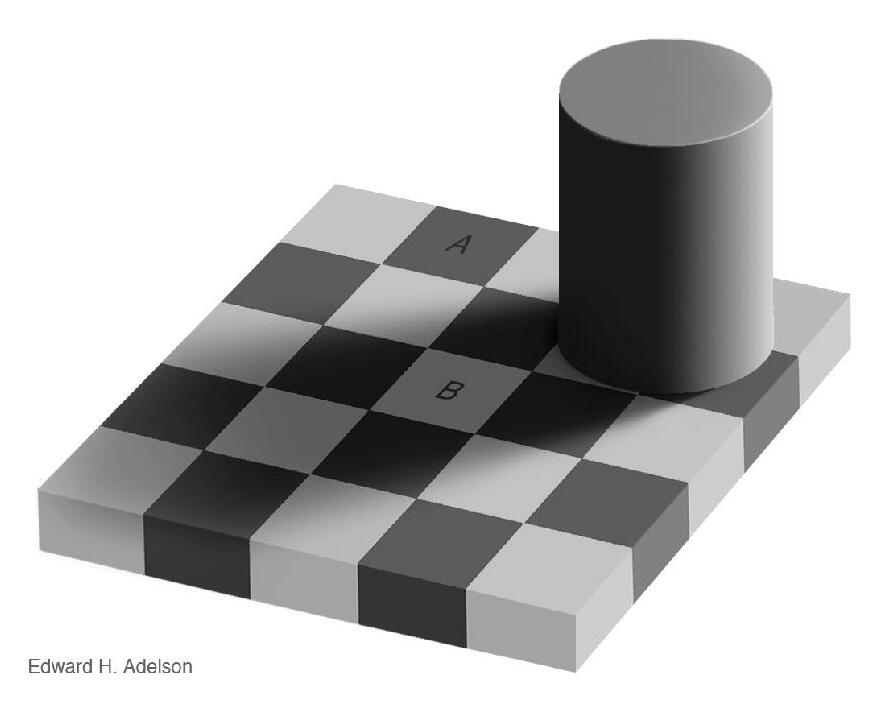}\label{fig:Im_cb}}
\subfloat[][\emph{Grayscale Adelson's checker \\ shadow illusion: the two squares}.]
{\includegraphics[height=5cm]{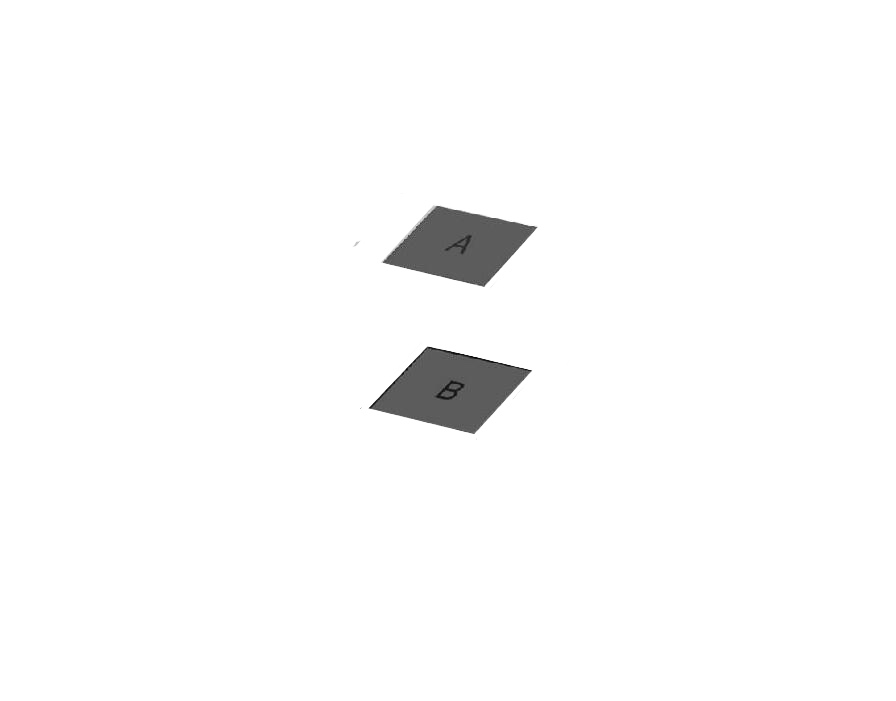}\label{fig:Im_cb_crop}}

\subfloat[][\emph{Retinex effects of LoG}.]
{\includegraphics[height=5cm]{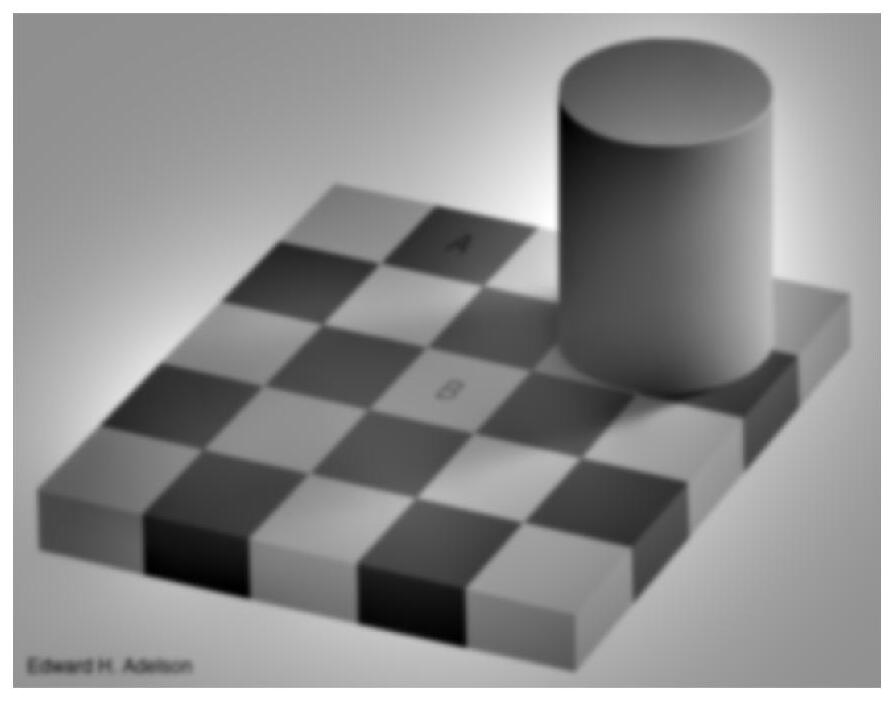}\label{fig:Im_LoG_cb}}
\subfloat[][\emph{Retinex effects of LoG: the two squares}.]
{\includegraphics[height=5cm]{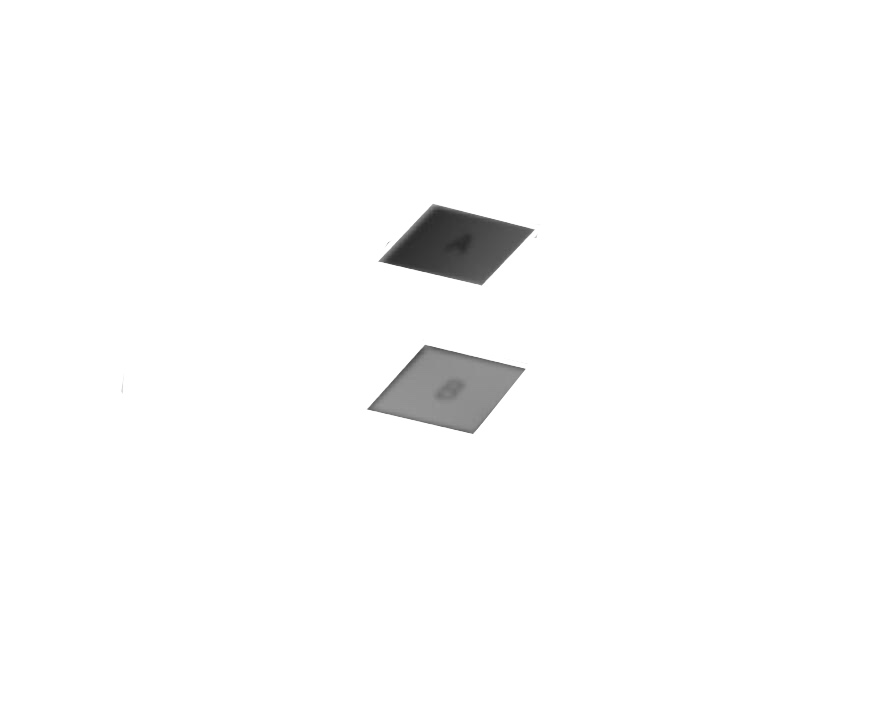}\label{fig:Im_LoG_cb_crop}}

\subfloat[][\emph{Retinex effects of $\Psi^0$ of LGN-CNN}.]
{\includegraphics[height=5cm]{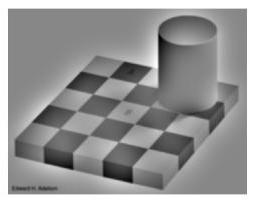}\label{fig:Im_Psi_cb}}
\subfloat[][\emph{Retinex effects of $\Psi^0$ of LGN-CNN: the two squares}.]
{\includegraphics[height=5cm]{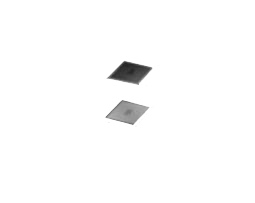}\label{fig:Im_Psi_cb_crop}}

\caption{Comparison between the Retinex effects of some operators on grayscale Adelson's checker shadow illusion.}
\label{fig:Comparison_Adelson_checker}

\end{figure*}

\subsection{Information transmission efficiency}

We are now interested to see if the layer $\ell^0$ have similar properties as regards the information transmission efficiency of the LGN. Indeed, it is well established (see e.g. \cite{Reinagel}, \cite{Zaghloul}, \cite{Uglesich}, \cite{Im}, \cite{Pregowska}) that the average firing rate of the retinal neurons that drive information to the LGN is much bigger than that of the LGN. In particular, LGN is able to delete spikes preserving the more informative ones leading to a loss of information. 

Thus, we have studied the information loss on the 8000 images of STL10 test set by convolving each image with $\Psi^0$ and computing the entropy from the histogram of gray scale values via the built-in MatLab function \textit{entropy}. It turns out that on average the entropy decreases from 7.04 to 5.97 with a loss of 15.27 \% of the information. Thus, we have reconstructed the images using the Retinex algorithm described in Section \ref{subs:Gen_appr_Retinex}. The average entropy increases to 6.92 leading to a loss of 1.83\% of the information w.r.t. the original dataset. This suggests that almost the entire information contained in the visual stimulus can be reconstructed via some feedback or horizontal connections, where the reconstructed stimulus becomes invariant w.r.t. lightness constancy. Figure \ref{fig:Ex_horse} shows an example of the convolution and the reconstruction performed on an image of the dataset with the corresponding gray values histogram with respect to the entropy is calculated.

\begin{figure}
\centering
\includegraphics[height=6.25cm]{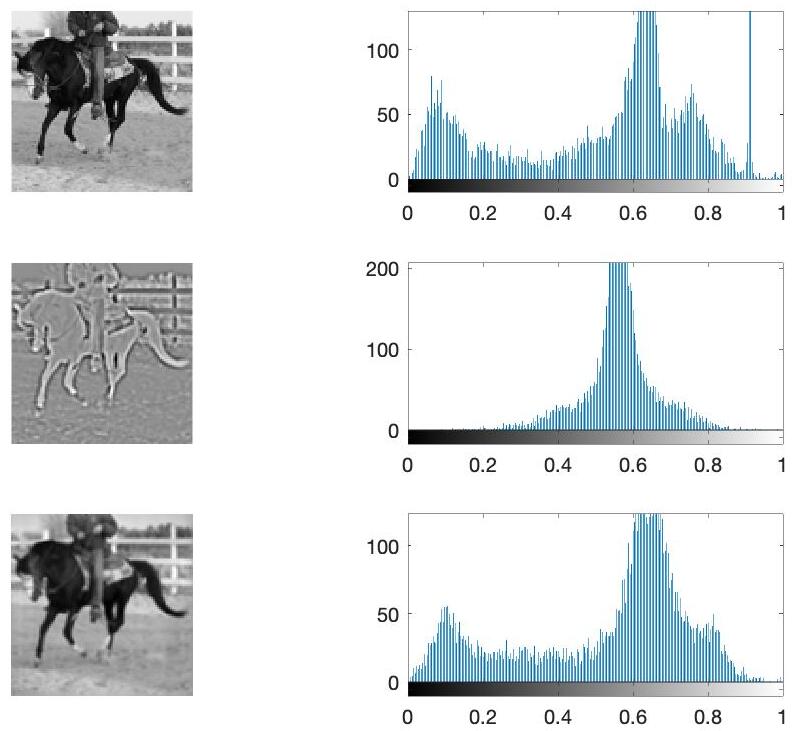}
\caption{In order on the left: a grayscale image $I$,  the convolved image $\widetilde{I} $ and the reconstructed image $\widetilde{I}$ via eq. (\ref{eq:PDE_Itilde}). On the right: the corresponding histograms of the grayscale values.}
\label{fig:Ex_horse}
\end{figure}

\section{Conclusions}          %crea il capitolo
%%%%%%%%%%%%%%%%%%%%%%%%%%%%%%%%%%%%%%%%%imposta l'intestazione di pagina
\lhead[\fancyplain{}{\bfseries\thepage}]{\fancyplain{}{\bfseries\rightmark}}
%\pagenumbering{arabic} 

The study of the role of the LGN in the visual system and the rotation invariance properties of the RFPs of its cells has leaded our research to the introduction of a CNN architecture that mimics this structure. In particular, we have added to a CNN a first convolutional layer composed by a single filter which attains a rotational symmetric pattern. The filter $\Psi^0$ has inherited this property from the modified architecture of the neural network. 

We have also shown that it is not only rotational symmetric but it also obtains a LoG shape that approximates the RFPs of the LGN cells. In order to study these similarities, we have shown that the Retinex effects of the LoG and $\Psi^0$ are really closed to each other. These behaviors enforce the link between the visual system structure and the architecture of CNNs.

Furthermore, we have analyzed the statistical distribution of the filters of the second convolutional layer that attain a Gabor shape even with the introduction of the first layer. We have shown that the statistical distribution becomes closer to the real data of RFPs of simple cells in V1 from \cite{Ringach} enriching the connections with the neural structure.

Then, we have faced the theoretical problem regarding the rotation symmetry of the first convolutional layer. We have studied the solution of a convex functional composed by a convolution and a ReLU. Thanks to uniqueness of such functional, we have shown that the solution $\Psi^0$ has to be rotational invariant. Then, we have built an architecture composed by a single filter $\Psi^0$ and a ReLU in which $\Psi^0$ has attained a rotational invariant pattern close to a Gaussian.

In the future we will face the theoretical problem regarding the rotation symmetry of the first convolutional layer for a general LGN-CNN. Furthermore, we will analyze the modifications that in an LGN-CNN occur to the bank of filters of other convolutional layers of deeper architecture, comparing them with neural data. We will also introduce an autoencoder associated to this architecture, which can reconstruct perceived images 
with the Retinex effect.

%%Era cosí nel paper su arxiv
%In the future we will face the theoretical problem regarding the rotation symmetry of the first convolutional layer. Furthermore, we will analyze the modifications that in a LGN-CNN occur to the bank of filters of other convolutional layers of deeper architecture, comparing them with neural data.

                          %imposta le appendici
\appendix\section{Appendix } \label{Third}            %crea il capitolo
%%%%%%%%%%%%%%%%%%%%%%%%%%%%%%%%%%%%%%%%%imposta l'intestazione di pagina
\lhead[\fancyplain{}{\bfseries\thepage}]{\fancyplain{}{\bfseries\rightmark}}
%\pagenumbering{arabic} 

\subsection{Rotation symmetry of $\Psi^0$} %\hfill\break

%{\color{green} attenzione alla numerazione dell'appendice}
%{\color{green} io la chiamerei appendice A}
In this section we define the setting in which we study the rotation symmetry of $\Psi^0$. %The idea is to extend some results on rotation invariant functionals (see \cite{Lopes}). 

%In this section we are going to study the rotation symmetry of $\Psi^0$. We will simplify the problem in order to use some results on rotation invariant functionals (see \cite{Lopes}). 

Let us consider the architecture of an LGN-CNN in which we can split the first convolutional layer composed by only one filter from the rest of the neural network which will be fixed. Thus, this first layer can be approximated by a function $\Psi^0: \R^2 \to \R$, assuming $\Psi^0 \in L^1_{loc} (\R^2)$. A general image can be defined as a function $I: \R^2 \to \R$ where we assume $I \in L^1_{loc} (\R^2)$. 
We can consider a subset $\Gamma \subset L^1_{loc} (\R^2)$ of all the images where for each image $I$ is defined a labelling $y  : \Gamma \to \R$, where $y(I)$ is the corresponding label to image $I$.

We require some rotational invariant properties on this set $\Gamma$. In particular, if we consider a rotation $R_\theta$ of an angle $\theta$ on $R^2$ plane around its center then the composition $I_\theta = R_\theta ( I) = I (R_{-\theta} (x))$ is still an image and we can also assume that  $I_\theta \in \Gamma $ (i.e., that the subset $\Gamma$ is close under rotation). 
%We use this notation to simplify calculus later, in which $R_\theta ( I)$ represents the rotation of an angle $\theta$ w.r.t. the $z$-axis in the 3 dimensional space of the image and $R_{-\theta} (x)$ is the rotation of an angle $-\theta$ in the 2 dimensional space of the domain of the image.
Furthermore, the rotated image should maintain the same label, i.e., $y(I_\theta) = y(I)$.
Since the images we will consider in this problem are grayscale ones (thus with values between 0 and 1) we can also assume that the images are normalized with $|| I ||_{L^1_{loc}} \leq 1$. Summarizing all these properties   we can assume that $\Gamma$ is a compact set on $L^1_{loc}$, $\Gamma = \{ I \in L^1_{loc} ; || I ||_{L^1_{loc}} \leq 1 \}$

Thus, the rest of the neural network will be defined by a nonlinear functional
\begin{equation} \label{eq:C}
\begin{split} 
\;\;\;\;\;\;\;\;\;\;\;\;\;\;\;C & : L^1_{loc} (\R^2)  \to L^1_{loc} (\R^2)  \\
C & (f(x)) = \max (0, f(x)) %\log (1 + e^{f(x)})
\end{split}
\end{equation}
% \;\;\;\; , \;\;\;\;  i = 1, \ldots d%C (g) =  (u \ast \phi )(x) . 
which is one of the more frequently nonlinear function used in CNN, called ReLU. %$R(x) = \max (0, x)$.
And then we can define
\begin{equation}\label{eq:F}
\begin{split}
\;\;\;\;\; \;\;\;\;\; F &:   L^1_{loc} (\R^2) \times L^1_{loc} (\R^2) \to  \R \\  F & (I, \Psi^0)   :=\int_{\R^2}C \left( ( I \ast \Psi^0 ) (z) \right)  dz.
\end{split}
\end{equation}
Then $F(I, \Psi^0) $ is the label that our architecture associates to the image $I$ and should eventually approximates $y(I)$.
%where each $F^i$ indicates the closeness of the image $I$ to the $i$-th category of the classification task. 

Thus, our aim is to find a function $\Psi^0$ in such a way that $F$ approximates well the known functional $y(I)$. %which is defined on a subset $\Gamma \subset L^1 (\R^2)$ of all the images.
In particular we would like to minimize the following functional
\begin{equation}
\label{eq:minPsi}
\min_{\Psi^0 \in L^1_{loc} (\R^2)} \int_\Gamma  | F(I, \Psi^0)  - y(I) |  d \mu(I)
\end{equation}

where the integral done over the set $\Gamma$ is a Bochner's integral (see e.g., Section 5 of chapter V of \cite{Yosida} and \cite{Mikusinski}).

Our aim is to find a function $\Psi^0$ that attains the minimum of (\ref{eq:minPsi}) where $F$ is defined in (\ref{eq:F}) and $C$ is defined in (\ref{eq:C}). We would like to find out if there exist some rotational invariant properties on the function $\Psi^0$. 
Let us note that the functional defined in (\ref{eq:minPsi}) is convex thanks to the convexity of the function $C$ defined in (\ref{eq:C}) and continuous. Then the existence and uniqueness of a solution is guaranteed (see e.g., Section 1.4 of \cite{Brezis}).

\begin{rem}
\label{remark:rotation}
Let us consider two function $f,g \in L^1 (\R^2)$ and a rotation $R_\theta$ of an angle $\theta$. Then 
$$  f \ast R_\theta (g)  (x) =  (R_{-\theta}(f) \ast  g)  (R_{-\theta}(x)) $$
\end{rem}

\begin{proof}
\begin{equation*}
\begin{split}
 f \ast R_\theta (g)  (x) = & \int_{\R^2} f(x-y)  R_\theta (g(y)) dy \\ = & \int_{\R^2} f(x-y)  g(R_{-\theta} (y)) dy =
\end{split}
\end{equation*}
%$$  f \ast R_\theta (g)  (x) = \int_{\R^2} f(x-y)  R_\theta (g(y)) dy = \int_{\R^2} f(x-y)  g(R_{-\theta} (y)) dy = $$
then we substitute $y' = R_{-\theta} (y)$ whose Jacobian has determinant equal to 1
\begin{equation*}
\begin{split}
\;\;\;\;\;\;\;\;\;\;\;\;\;\;\;\;\;\;\;\; = & \int_{\R^2} f(x-R_\theta (y'))  g(y') dy' \\ = &  \int_{\R^2} f(R_\theta(R_{-\theta}(x))-R_\theta (y'))  g(y') dy'  \\  
 = & \int_{\R^2} f(R_\theta(R_{-\theta}(x)-y'))  g(y') dy'  \\ = & \int_{\R^2} R_{-\theta}(f(R_{-\theta}(x)-y'))  g(y') dy'  =
\end{split}
\end{equation*}
%$$  = \int_{\R^2} f(x-R_\theta (y'))  g(y') dy' =  \int_{\R^2} f(R_\theta(R_{-\theta}(x))-R_\theta (y'))  g(y') dy'  = $$
%$$ = \int_{\R^2} f(R_\theta(R_{-\theta}(x)-y'))  g(y') dy'  = \int_{\R^2} R_{-\theta}(f(R_{-\theta}(x)-y'))  g(y') dy'  =  $$
then we substitute $y'' = R_{-\theta} (x) - y'$ whose Jacobian has determinant equal to 1
\begin{equation*}
\begin{split}
\;\;\;\;\;\;\;\;\;\;\;\;\;\;\;\;\;\;\;\; = & \int_{\R^2} R_{-\theta}(f(y''))  g(R_{-\theta}(x) - y'') dy''  \\ = & (R_{-\theta}(f) \ast  g)  (R_{-\theta}(x))
\end{split}
\end{equation*}
%$$ = \int_{\R^2} R_{-\theta}(f(y''))  g(R_{-\theta}(x) - y'') dy''  = (R_{-\theta}(f) \ast  g)  (R_{-\theta}(x)) $$
and this conclude the proof. 
\end{proof}

Now we can demonstrate the rotational invariance of $\Psi^0$.

\begin{teo}
\label{th:rotationinv}
Let $\Psi^0$ be a solution to the problem (\ref{eq:minPsi}) where $F$ is defined in (\ref{eq:F}) and $C$ is defined in (\ref{eq:C}). Then $\Psi^0$ is rotational invariant.

\end{teo}

\begin{proof}
Let us consider the rotated solution $R_\theta(\Psi^0)$ of an angle $\theta \in [0, 2\pi]$. 

$$  \int_\Gamma  | \int_{\R^2}C \left( ( I \ast R_\theta(\Psi^0) ) (z) \right)  dz  - y(I) |  d \mu(I)  =$$
and because of remark \ref{remark:rotation} 
$$  =\int_\Gamma  | \int_{\R^2}C \left( ( R_{-\theta}(I) \ast \Psi^0 ) (R_{-\theta}(z)) \right)  dz  - y(I) |  d \mu(I)  $$
Since for a general $f \in L^1 (\R^2)$ it holds
$$  \int_{\R^2} f (R_{\theta}(x)) dx    = \int_{\R^2} f(x) dx,$$ then 
$$  =\int_\Gamma  | \int_{\R^2}C \left( ( R_{-\theta}(I) \ast \Psi^0 ) (z) \right)  dz  - y(I) |  d \mu(I)  =$$
Finally, since $\Gamma$ is close under rotation and $y(I) = y(R_{-\theta}(I))$ by hypothesis 
$$  =\int_\Gamma  | \int_{\R^2}C \left( ( (I) \ast \Psi^0 ) (z) \right)  dz  - y(I) |  d \mu(I)  $$
Thus, $R_\theta(\Psi^0)$ attains the same value of $\Psi^0$ for every choice of $\theta$. But thanks to the uniqueness of the solution of this problem because of the compactness of $\Gamma$ and convexity of the functional (\ref{eq:minPsi}), $\Psi^0=R_\theta(\Psi^0)$ $\forall \theta \in [0, 2\pi]$ and this concludes the proof.
\end{proof}

\begin{rem}
Let us note that the first part of the proof of theorem \ref{th:rotationinv} is valid for a general LGN-CNN architecture. In particular, if we have a solution $\Psi^0$ to the minimization problem (\ref{eq:minPsi}), then every rotation $R_{\theta} (\Psi^0)$ of an angle $\theta \in [0, 2 \pi]$ is still a solution. The uniqueness of solution in theorem \ref{th:rotationinv} guarantees that $\Psi^0$ is rotational invariant whereas for a general LGN-CNN this does not hold.
\end{rem}
\subsection{Testing the theorem on the same architecture}% \hfill\break

\begin{figure}[ht]
\centering
\includegraphics[height=4cm]{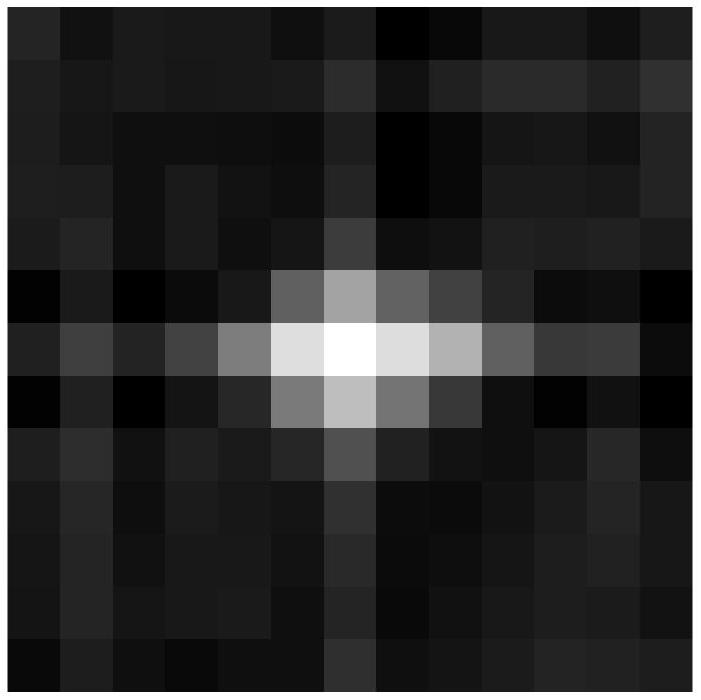} %14
\caption{Filter $\Psi^0$ of LGN-CNN obtains after training.}
\label{fig:LGN_CNN_simple_prob_Psi0}
\end{figure}

In this Section we  face the same problem of the proof in Section \ref{Third} to see if $\ell^0$ becomes rotational invariant. Since we are going to use a simple architecture composed by a single convolutional layer $\ell^0$ with a single filter $\Psi^0$ and a ReLU we do not expect to obtain a LoG shape filter as in Section \ref{Second}. Indeed, in the previous case the architecture of LGN-CNN has other convolutional layers whose aim was to further analyze the image and in particular the contours of the objects. For this reason we  expected in that case that the LGN-CNN behaved similarly to the LGN and the V1, i.e., $\Psi^0$ had a LoG shape. On the other hand, in the test we are performing now we can only expect a rotational invariant filter as stated in theorem \ref{th:rotationinv}.

 \begin{figure*}[ht]
\centering
%\subfloat[][\emph{Filter $\Psi^0$ of first layer of our architecture}.] {\includegraphics[height=5cm]{CNN_1st_filter-2.png}\label{fig:comparison_filterA}} 
%\hspace{3mm}
\subfloat[][\emph{Filter $\Psi^0$ of first layer of LGN-CNN}.]
{\includegraphics[height=5cm]{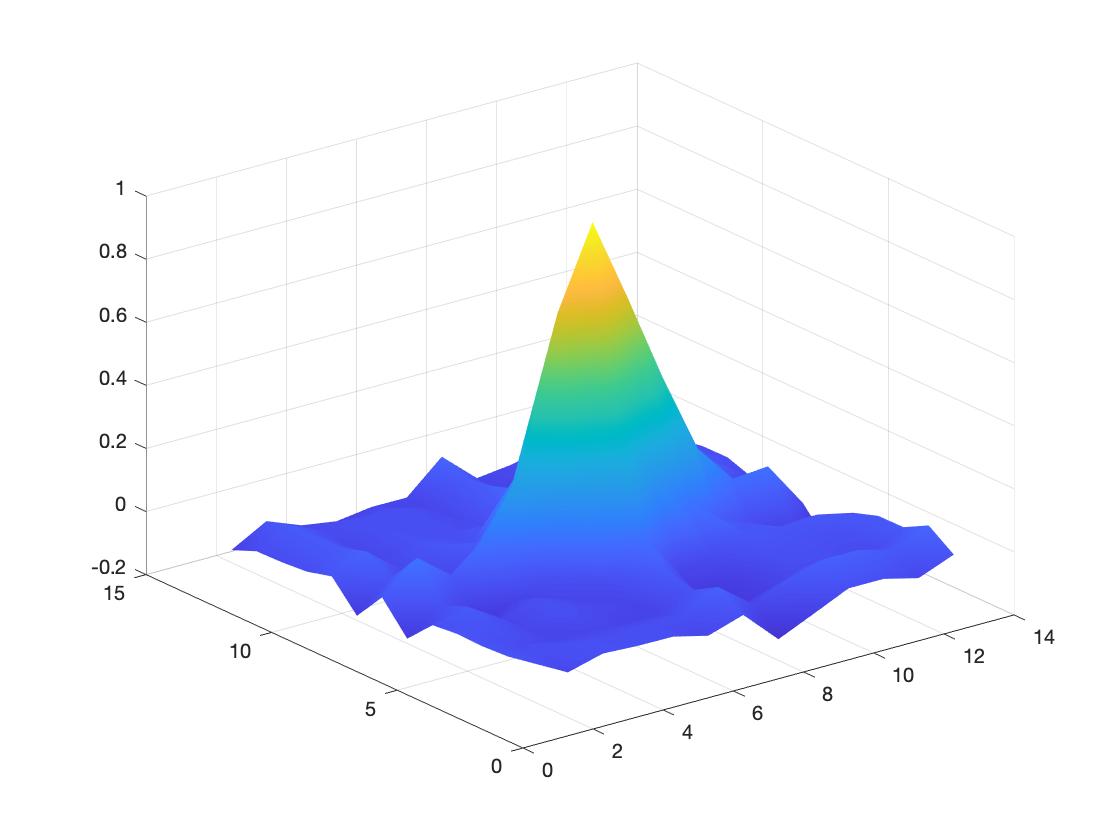}\label{fig:comparison_filterA_mnist}}
%\hspace{3mm}
\subfloat[][\emph{Gaussian function}.]
{\includegraphics[height=5cm]{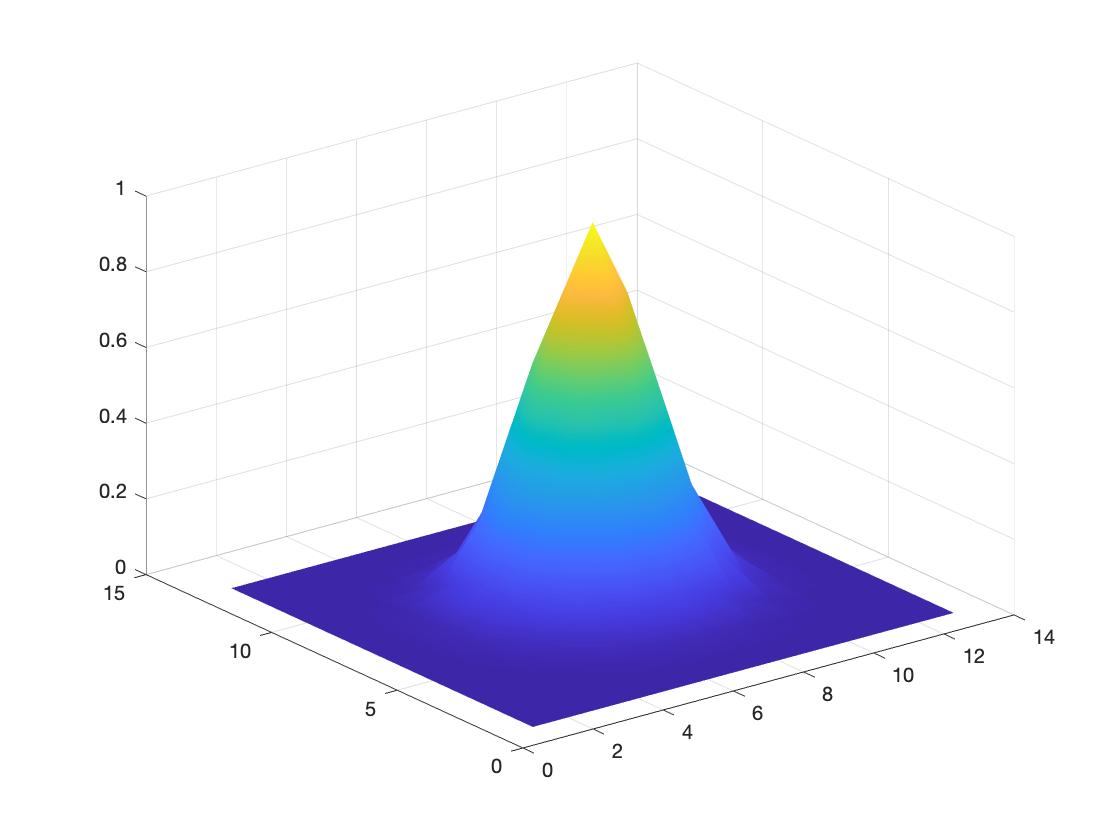}\label{fig:comparison_filterB_mnist}}
\caption{Comparison between the filter $\Psi^0$ and a Gaussian. We can see that $\Psi^0$ is really close to a discrete approximation of a Gaussian fitted to the data.}
\label{fig:comparison_filter_mnist}

\end{figure*}

%\hfill\break
To perform this test we have built a new dataset of images starting from the dataset MNIST (a set of digits images, see \cite{Lecun}) and the dataset Fashion-MNIST (a set of cloths images, see \cite{Xiao}). They are two similar datasets, composed by grayscale images of size $28 \times 28$.
The aim of our LGN-CNN is to classify the input as a digit or a cloth, indeed if the image belongs to MNIST or Fashion-MNIST dataset. The new training set has been built by taking the first half of MNIST dataset (30000 images) and the first half of Fashion-MNIST dataset (30000 images), for a total of 60000 images. We have followed the same steps for the test set (for a total of 10000 images) and we have randomly sorted the training and test sets. Then each image has been labeled by 0 if it is a digit and by 1 if it is a cloth.

We have built a really simple LGN-CNN architecture that contains only a first layer with a single filter $\Psi^0$ of size $13 \times 13$ followed by a ReLU and by a fully connected layer. We have trained this neural network for a total of 25 epochs obtaining an accuracy of 98.55 \% on the classification task. 
Figure \ref{fig:LGN_CNN_simple_prob_Psi0} shows $\Psi^0$ of this LGN-CNN architecture. We can observe that it has a rotational invariant shape as we expected from theorem \ref{th:rotationinv}. Then we have tried to approximate the filter $\Psi^0$ with a Gaussian with the following formula 
$  G(x,y) = \alpha e^{\frac{-x^2-y^2}{2 \sigma^2}}. $
Figure \ref{fig:comparison_filter_mnist} shows $\Psi^0$ and its approximation by a Gaussian. The rotation invariance of $\Psi^0$ is now enforced thanks to the approximation obtained with a rotational invariant function as the Gaussian.

%\input{Appendice.tex}

%% Non-BibTeX users please use
%\begin{thebibliography}{}
%%
%% and use \bibitem to create references. Consult the Instructions
%% for authors for reference list style.
%\end{thebibliography}

\end{document}